\documentclass[11pt,oneside]{article}

\usepackage[T1]{fontenc} 
\usepackage{textcomp} 
\usepackage{lmodern} 
\DeclareTextCommand{\nobreakspace}{T1}{\leavevmode\nobreak\ } 
\usepackage{amsmath} 
\usepackage{amssymb,amsfonts} 
\usepackage{bm} 

\usepackage[paper=a4paper, hscale=0.75, vscale=0.79]{geometry} 

\usepackage{graphicx} 
\graphicspath{{./img/}} 
\DeclareGraphicsExtensions{.pdf,.png,.jpg,.eps} 

\usepackage[dvipsnames]{xcolor} 

\usepackage{algorithm}
\usepackage{algorithmic}

\usepackage{booktabs} 
\usepackage{multirow} 

\usepackage{natbib}
\usepackage{amsmath} 
\usepackage{amssymb} 
\usepackage{amsfonts} 
\usepackage{bm} 
\usepackage{amsthm} 
\theoremstyle{definition} \newtheorem{defn}{Definition}
\theoremstyle{plain} \newtheorem{prop}[defn]{Proposition}
\theoremstyle{plain} \newtheorem{thm}[defn]{Theorem}
\theoremstyle{plain} 
\theoremstyle{plain} 
\theoremstyle{remark} \newtheorem{rmk}[defn]{Remark}
\theoremstyle{remark}

\newcommand{\term}[1]{\textcolor{BlueViolet}{\textit{{#1}}}}

\newcommand*{\defeq}{\mathrel{\vcenter{\baselineskip0.5ex \lineskiplimit0pt     
                     \hbox{\scriptsize.}\hbox{\scriptsize.}}}=}

\newcommand{\overbar}[1]{\mkern 1.5mu\overline{\mkern-1.5mu#1\mkern-1.5mu}\mkern 1.5mu}

\DeclareMathOperator*{\argmin}{arg\,min}

\newcommand{\Abs}[1]{\lVert{#1}\rVert} 
\newcommand{\abs}[1]{\lvert{#1}\rvert} 
\DeclareMathOperator{\crit}{C} 
\def\colonset{:\,} 
\def\cond{\,\vert\,} 
\def\ddist{\mu} 
\newcommand{\err}[1]{\mathcal{E}({#1})} 
\def\errMin{\mathcal{E}^{\ast}} 
\def\exx{\mathbf{E}} 
\def\Herr{\mathcal{H}_{\mathcal{E}}^{\ast}} 
\def\HH{\mathcal{H}} 
\DeclareMathOperator{\indic}{I} 
\def\LL{\mathcal{L}} 
\DeclareMathOperator{\loss}{\rdv{L}} 
\newcommand{\prr}[1]{\mathbf{P}{#1}} 
\newcommand{\rdv}[1]{\mathsf{#1}} 
\def\RR{\mathbb{R}} 
\DeclareMathOperator{\sign}{sign} 
\def\XX{\mathcal{X}} 
\def\YY{\mathcal{Y}} 

\usepackage[pdfauthor={Matthew J. Holland},%
colorlinks=true,%
linkcolor=blue,%
citecolor=blue]{hyperref}
\hypersetup{pdftitle={Criterion Collapse and Loss Distribution Control}} 

\begin{document}

\title{\textbf{Criterion Collapse and Loss Distribution Control}}

\author{
  Matthew J.~Holland\\
  Osaka University
}
\date{} 

\maketitle

\begin{abstract}
In this work, we consider the notion of ``criterion collapse,'' in which optimization of one metric implies optimality in another, with a particular focus on conditions for collapse into error probability minimizers under a wide variety of learning criteria, ranging from DRO and OCE risks (CVaR, tilted ERM) to non-monotonic criteria underlying recent ascent-descent algorithms explored in the literature (Flooding, SoftAD). We show how collapse in the context of losses with a Bernoulli distribution goes far beyond existing results for CVaR and DRO, then expand our scope to include surrogate losses, showing conditions where monotonic criteria such as tilted ERM cannot avoid collapse, whereas non-monotonic alternatives can.
\end{abstract}

\tableofcontents

\clearpage

\section{Introduction}\label{sec:intro}

As machine learning systems become more widespread and integrated in our daily lives, the tension between performance metrics we ideally wish to optimize (at test time) and those we actually optimize in practice (at training time) becomes increasingly important. Broadly speaking, tasks are characterized by \emph{losses} and what we will call \emph{learning criteria}. Losses are diverse, depending on the underlying problem to be solved (e.g., classification, regression, ranking, clustering, etc.) and a wide range of application-specific needs. Losses depend on random data, and are themselves random. In principle, performance can be described in terms of random loss distributions, but working with complex and unwieldy distributions is not practical; a concise numerical summary is usually desired both for designing objectives to optimize, as well as for disseminating interpretable results. Providing such a summary is the role of learning criteria. In contrast with the diversity of losses, there is a single \textit{de facto} standard learning criterion, namely the expected value of the loss distribution, a criterion which is central to the ``general setting of the learning problem'' \citep{vapnik1999NSLT} that is pervasive in both theory and practice.

In recent years, however, the need to capture a wider variety of performance metrics at training time has motivated the use of novel learning criteria going beyond the expected value. Arguably the best-known families of such criteria are those which control sensitivity to the right tail of the loss distribution, e.g., CVaR \citep{curi2020a}, positive-tilted ERM \citep{li2021a}, and Cressie-Read DRO \citep{duchi2021a}, though criteria designed to capture both tails at once have also begun to appear \citep{holland2022c}. While all of these criteria can be clearly shown to be distinct from the expected value, when such criteria are optimized at training time, can we expect the outcome to actually change? This question is somewhat delicate, since it depends on the nature of the loss distribution. For example, when we are working with the zero-one loss for classification, or any binary performance indicator for that matter, previous work has noted that vanilla empirical risk minimization (ERM) is always sufficient for minimizing DRO and CVaR \citep{hu2018a,zhai2021b}, making it meaningless to distinguish between optimizing such criteria.

What about when we make use of surrogate loss functions for training? The traditional theory of transferring consistency from a surrogate to a target loss is tied tightly to the expected value criterion (and its linearity) \citep{bartlett2006b,reid2011a}, meaning that a change in criterion can drastically change the relations between loss distributions. For example, given a model which is CVaR-optimal in terms of a surrogate loss, it may or may not be CVaR-optimal in terms of the zero-one loss; understanding when such properties hold is critical to ensure the learning algorithms we use in practice are aligned with our true objectives. For many classes of learning criteria, however, such questions remain completely unexplored.

In this work, we consider the notion of ``criterion collapse,'' in which optimization of one metric implies optimality in another, with a particular focus on conditions for collapse into error probability minimizers (i.e., minimizers of the expected zero-one loss) under a wide variety of learning criteria, including learning scenarios that involve a surrogate loss. We start in \S{\ref{sec:collapse}} by showing how collapse in the context of losses with a Bernoulli distribution goes far beyond CVaR and DRO. We then expand our scope to include surrogate losses in \S{\ref{sec:surrogates}}. First, we highlight cases where optimality diverges across loss distributions under a common criterion (\S{\ref{sec:surrogates_nolink}}), and then on the other hand show conditions where for \emph{monotonic} criteria (e.g., positive-tilted ERM), collapse into error probability minimizers is unavoidable (\S{\ref{sec:surrogates_unavoidable}}), though the use of non-monotonic alternatives can be used to avoid this when such collapse is not desired (\S{\ref{sec:surrogates_flood_softad}}). The main limitation of our results in \S{\ref{sec:surrogates}} is that we only cover margin-type losses for binary classification. Extensions to asymmetric, multi-class surrogate losses is left as future work. We complement our basic theory with a set of experiments in \S{\ref{sec:empirical}}, training non-linear neural network models (e.g., ResNet-34) for image classification from scratch, comparing across a variety of learning criteria with a common base loss. We find that when criterion selection is optimized for validation accuracy, the non-monotonic criteria (with clear links to ascent-descent learning algorithms) provide an appealing balance across other metrics such as model norm and average (surrogate) loss at test time.

\section{Criterion Collapse}\label{sec:collapse}

Underlying this section is the following simple question: \emph{is it meaningful to introduce a new criterion?} Even if two criteria are distinct in a mathematical sense, i.e., they return different values for the same loss distribution, if the solution set of one criterion completely includes the other, then an obvious \emph{sufficiency} property holds; no tradeoffs between criteria arise. When we have a sufficient condition (for optimality in both criteria) that is easy to satisfy, it essentially renders one criterion meaningless. Here we will look at how such a phenomenon occurs frequently under binary distributions.

\subsection{Preliminaries}

\paragraph{Basic notation}
Let $\XX \times \YY$ denote our underlying data space. Random data points will be denoted by individual random pairs distributed on $\XX \times \YY$. For a single data point, often to represent ``data at test time,'' we write $(\rdv{X},\rdv{Y})$. When we are dealing with non-random quantities, $\rdv{X}$ and $\rdv{Y}$ will be replaced by $x$ and $y$ respectively. As general-purpose notation for decisions, we write $h$, assumed to be an element of a set of admissible decisions $\HH$ (the ``hypothesis class''). Throughout this paper, when we use $\exx$ and $\mathbf{P}$ to represent expectation and probability, it will always be with respect to the distribution of $(\rdv{X},\rdv{Y})$, unless otherwise noted. For an arbitrary real-valued function $f: \HH \to \RR$, we denote the set of minimizers by $\argmin_{h \in \HH} f(h) \subset \HH$, with $\argmin_{h \in \HH} f(h) = \emptyset$ in the case that no minimizers exist in $\HH$. We remark that unlike $h$ (always $h:\XX \to \YY$), the notation $f$ is not reserved here. We use $f$ and also $g$ to denote various different helper functions throughout the paper.

\paragraph{Loss function and criterion mapping}
Central to this paper is the notion of a numerical \emph{loss} that can be used to provide feedback to a learning algorithm, or as a metric for evaluation after learning is complete. A \term{loss function} $\ell: \HH \times \XX \times \YY \to \RR$ maps (decision, data point) pairs to real values $\ell(h;x,y)$, called losses. When the data is random, so is the loss. We denote random losses by $\loss(h) \defeq \ell(h;\rdv{X},\rdv{Y})$, making the dependence on decision $h \in \HH$ explicit in our notation. Running over the set $\HH$, we end up with a set of random losses, denoted by $\LL \defeq \{\loss(h) \colonset h \in \HH \}$. We will sometimes just refer to individual random losses $\loss \in \LL$ without specifying which $h \in \HH$ is associated with $\loss$. The other basic notion we require is that of a \term{criterion mapping}, denoted $\crit: \LL \to \RR$, which maps random losses $\loss \in \LL$ to numerical values $\crit(\loss)$, called criteria. As mentioned in \S{\ref{sec:intro}, the traditional criterion is $\crit(\loss) = \exx[\loss]$, often called the ``risk,'' but we will go well beyond this criterion in the following sub-sections.

\subsection{Random error and criterion collapse}\label{sec:collapse_existing}

The most basic and pervasive loss function used for evaluation is the ``zero-one loss,'' a simple classification error penalty that we denote as
\begin{align}\label{eqn:defn_zeroone}
\ell_{\textup{01}}(h;x,y) \defeq
\begin{cases}
1, & \text{ if } h(x) \neq y\\
0, & \text{ otherwise}.
\end{cases}
\end{align}
Denoting the random error by $\loss_{\textup{01}}(h) \defeq \ell_{\textup{01}}(h;\rdv{X},\rdv{Y})$, it follows a Bernoulli distribution that is completely determined by its expected value, the error probability $\exx[\loss_{\textup{01}}(h)] = \prr\{ h(\rdv{X}) \neq \rdv{Y} \}$. Since this quantity will appear frequently, we reserve the following symbols for the error probability and its solution set:
\begin{align}\label{eqn:defn_err_Herr}
\err{h} \defeq \prr\{ h(\rdv{X}) \neq \rdv{Y} \}, \qquad \Herr \defeq \argmin_{h \in \HH} \err{h}.
\end{align}
The expected value $\loss \mapsto \exx[\loss]$ also happens to be the traditional choice of criterion mapping in machine learning, but as discussed earlier in \S{\ref{sec:intro}}, a variety of new choices for criterion mapping have arisen in the context of ``risk-sensitive'' learning.

In general, by introducing new criteria (instead of the expected value), one gains an additional degree of freedom (beyond choice of loss function) in terms of what qualities we evaluate, and how we quantify such qualities. However, for loss functions such as $\ell_{\textup{01}}$ in (\ref{eqn:defn_zeroone}), the resulting random loss $\loss_{\textup{01}}$ has such a simple distribution that many criteria ``collapse'' into the expected value. Two important special cases of this phenomenon have been noted in the previous literature. We recall these two cases here in chronological order, adapted to our notation for readability.
\begin{thm}[DRO criterion; \citeauthor{hu2018a}~(\citeyear{hu2018a}, Thm.~1)]\label{thm:collapse_DRO}
For arbitrary random loss $\loss \in \LL$, denote the distributionally robust optimization (DRO) criterion by
\begin{align*}
\textup{DRO}(\loss) \defeq \sup_{\ddist \in \mathcal{P}} \exx_{\ddist}[\loss]
\end{align*}
where the ``uncertainty set'' $\mathcal{P}$ is taken to be a ball centered at some pre-defined data distribution on $\XX \times \YY$, with finite radius measured by a valid $f$-divergence. Under zero-one loss $\loss = \loss_{\textup{01}}$, error probability minimizers are always optimal in terms of the DRO criterion, namely we have
\begin{align*}
\Herr \subset \argmin_{h \in \HH} \textup{DRO}(\loss_{\textup{01}}(h))
\end{align*}
with $\Herr$ defined earlier in (\ref{eqn:defn_err_Herr}).
\end{thm}
\noindent Another very closely related insight has been presented in the literature, this time looking at conditional value-at-risk (CVaR), a well-studied criterion that is central to the quantification of financial risk.
\begin{thm}[CVaR criterion; \citeauthor{zhai2021b}~(\citeyear{zhai2021b}, Prop.~1)]\label{thm:collapse_CVaR}
Again taking any loss $\loss \in \LL$, we write the conditional value-at-risk (CVaR) criterion as
\begin{align*}
\textup{CVaR}(\loss) \defeq \inf_{\theta \in \RR}\left[ \theta + \frac{1}{1-\beta} \exx[(\loss - \theta)_{+}] \right]
\end{align*}
where $0 < \beta < 1$ controls the degree of right-tail sensitivity, and $(\cdot)_{+} \defeq \max\{0,\cdot\}$. Under the zero-one loss $\loss = \loss_{\textup{01}}$, error probability minimizers are optimal in terms of CVaR, i.e.,
\begin{align*}
\Herr \subset \argmin_{h \in \HH} \textup{CVaR}(\loss_{\textup{01}}(h)).
\end{align*}
\end{thm}

The basic message of Theorems \ref{thm:collapse_DRO} and \ref{thm:collapse_CVaR} is clear: under losses with a Bernoulli distribution, it is essentially meaningless to consider DRO and CVaR as distinct from the expected value, since minimizing the expected loss is always sufficient for optimality in terms of DRO and CVaR as well. In \S{\ref{sec:collapse_broad}} to follow, we show that this ``collapse'' into error probability minimizers arises for criteria going well beyond DRO and CVaR.

\subsection{Which classes lead to collapse?}\label{sec:collapse_broad}

Here we take a look at several large classes of criteria, showing how most collapse in the sense illustrated in \S{\ref{sec:collapse_existing}}.

\subsubsection{Expectation of fixed function}\label{sec:collapse_broad_fixed}

Starting with the simplest class, we consider criteria that are computed as
\begin{align}
\loss \mapsto \exx[f(\loss)]
\end{align}
where $f:\RR \to \RR$ is any function such that the expectation is finite. By setting $\loss = \loss_{\textup{01}}$ and reflecting dependence on $h \in \HH$, note that
\begin{align*}
\exx[f(\loss_{\textup{01}}(h))] = f(0) + \err{h}\left(f(1)-f(0)\right)
\end{align*}
where $\err{\cdot}$ is the error probability defined in (\ref{eqn:defn_err_Herr}). If $f(0)=f(1)$ happens to hold, the criterion is constant and thus not very interesting. When $f(0) \neq f(1)$, minimizers of $\exx[f(\loss_{\textup{01}}(\cdot))]$ are easy to characterize: they either minimize or maximize $\err{\cdot}$, depending on whether $f(1) > f(0)$ or not. As such, no matter how $f$ is designed, the result of minimizing such a criterion will always be one of these two extremes.

\subsubsection{Quantiles}\label{sec:collapse_broad_quantiles}

As a natural alternative to the mean, let us consider a standard definition of quantiles, namely the $\beta$-level ``left quantile'' of the random loss, defined for $\loss \in \LL$ by
\begin{align}\label{eqn:defn_quantiles}
\mathrm{Q}_{\beta}(\loss) \defeq \min\{x \in \RR: \prr{\{\loss \leq x\}} \geq \beta\}
\end{align}
for all $0 < \beta \leq 1$, and $\mathrm{Q}_{\beta}(\loss) \defeq -\infty$ for $\beta = 0$. While in general the relation between quantiles and the mean can be very complicated, for the special case of Bernoulli losses, this relation is very simple for all choices of $\beta$, and we can readily show how the quantile criterion collapses to the mean.
\begin{prop}[Collapse of left quantiles]\label{prop:collapse_quantiles}
Given $\mathrm{Q}_{\beta}$ as defined in (\ref{eqn:defn_quantiles}) and $\loss = \loss_{\textup{01}}$, we have
\begin{align*}
\Herr \subset \argmin_{h \in \HH} \mathrm{Q}_{\beta}(\loss_{\textup{01}}(h))
\end{align*}
for all probability levels $0 \leq \beta \leq 1$.
\end{prop}
\begin{rmk}[Related case: right quantiles]
Another very similar discussion can be carried out with the ``right quantiles'' defined by
\begin{align}
\mathrm{Q}_{\beta}^{-}(\loss) \defeq \max\{x \in \RR: \prr{\{\loss < x\}} \leq \beta\}
\end{align}
for $0 \leq \beta < 1$, and $\mathrm{Q}_{\beta}^{-}(\loss) \defeq \infty$ for $\beta = 1$.
\end{rmk}

\subsubsection{Distribution dependent functions}\label{sec:collapse_broad_ocelike}

As a natural extension to the rudimentary ``fixed function expectation'' criteria seen in \S{\ref{sec:collapse_broad_fixed}}, one can consider a family of functions $f(\cdot;\theta)$ parameterized by $\theta \in \RR$, for which the value of $\theta$ is allowed to be determined based on the $\loss \in \LL$ being evaluated. While countless possibilities exist, one class of criteria that captures many important special cases in the literature is to take $f(u;\theta) = \theta + \rho(u-\theta)$, and to \emph{optimize} with respect to $\theta$ after taking expectation. As a typical example, consider
\begin{align}\label{eqn:defn_oce_like}
\underline{\crit}_{\rho}(\loss) \defeq \inf_{\theta \in \RR}\left[ \theta + \exx[\rho(\loss-\theta)] \right]
\end{align}
where $\rho: \RR \to \RR$ used in (\ref{eqn:defn_oce_like}) is assumed to be such that for each $\loss \in \LL$, the function $\theta \mapsto \theta + \exx[\rho(\loss-\theta)]$ achieves its minimum on $\RR$, and that $\inf\{\underline{\crit}_{\rho}(\loss) \colonset \loss \in \LL\} > -\infty$. As we will discuss later, special cases of the learning criterion $\underline{\crit}_{\rho}(\cdot)$, as well as close variants, arise frequently in the literature, particularly in the case where $\rho$ is monotonic (non-decreasing) and convex on $\RR$. As the following result shows, under the zero-one loss, monotonicity alone is sufficient to imply collapse into error probability minimizers.
\begin{prop}[Collapse under monotonic dispersion]\label{prop:collapse_broad_ocelike}
With $\underline{\crit}_{\rho}$ as defined in (\ref{eqn:defn_oce_like}), we have
\begin{align*}
\Herr \subset \argmin_{h \in \HH} \underline{\crit}_{\rho}(\loss_{\textup{01}}(h))
\end{align*}
when $\rho(\cdot)$ is non-decreasing, and equality when $\rho(\cdot)$ is increasing.
\end{prop}
\noindent%
As we note in the following remarks, a wide range of well-known criteria are captured by Proposition \ref{prop:collapse_broad_ocelike}, or can be shown to collapse using an analogous argument.
\begin{rmk}[Special case: OCE criteria]\label{rmk:oce_risks}
When $\rho$ in (\ref{eqn:defn_oce_like}) is assumed to be a closed, convex, sub-differentiable function, normalized such that $\rho(0)=0$ and $1 \in \partial\rho(0)$, which is monotonically non-decreasing on $\RR$, the resulting learning criterion is called an \term{optimized certainty equivalent (OCE)} (see also Figure \ref{fig:demo_valid_rho_rhotilde}). Clearly, setting $\rho(u) = u$ recovers the expected value. Another well-known special case of OCE criteria is CVaR, which is recovered by setting $\rho(u) = \max\{0,u\}/(1-\beta)$ with $0 < \beta < 1$. Clearly this $\rho$ is non-decreasing, so the inclusion in Proposition \ref{prop:collapse_broad_ocelike} holds; note that Theorem \ref{thm:collapse_CVaR} is thus implied. On the other hand, the $\rho$ used for CVaR is not strictly increasing. Under the zero-one loss, are there CVaR-optimal solutions that are \emph{not} optimal in terms of error probability? This is indeed possible, but it depends on both $\HH$ and $\beta$; we discuss this in detail in \S{\ref{sec:cvar_01_diverge}}. Another well-studied special case from the OCE class is the \term{tilted risk} (also \term{entropic risk}), recovered by setting $\rho(u) = (\mathrm{e}^{\gamma u}-1)/\gamma$ with $\gamma > 0$, taking the form
\begin{align}\label{eqn:tilted_risk_identity}
\underline{\crit}_{\rho}(\loss) = \frac{1}{\gamma} \log\left(\exx\left[\mathrm{e}^{\gamma\loss}\right]\right).
\end{align}
Note that since the exponential function is monotonically increasing, so is $\rho$, and thus by Proposition \ref{prop:collapse_broad_ocelike}, tilted risk minimization under $\loss = \loss_{\textup{01}}$ is identical to $\err{\cdot}$ minimization. To conclude this remark, we note that more generally, the monotonicity of all OCE risks immediately implies that all $\err{\cdot}$-optimal rules are also $\underline{\crit}_{\rho}$-optimal, even without the other assumptions of convexity and normalization.
\end{rmk}
\begin{rmk}[Related case: Cressie-Read DRO]\label{rmk:dro_cressieread}
An important class of \term{distributionally robust optimization (DRO)} criteria is defined in a very similar way to the OCE class given in (\ref{eqn:defn_oce_like}). In particular, using the Cressie-Read family of $f$-divergences leads to criteria of the form
\begin{align*}
\textup{DRO}_{c,\varepsilon}(\loss) \defeq \inf_{\theta \in \RR} \left[ \theta + \left(\exx\left[\rho_{c,\varepsilon}(\loss_{\textup{01}}(h)-\theta)\right]\right)^{1/c_{\ast}} \right]
\end{align*}
where we have defined $\rho_{c,\varepsilon}(x) \defeq (1+c(c-1)\varepsilon)^{c_{\ast}/c}(x)_{+}^{c_{\ast}}$ and $c_{\ast} \defeq c/(c-1)$, and these criteria are parameterized by $c > 1$ and $\varepsilon \geq 0$. Note that the function $\rho_{c,\varepsilon}$ is clearly non-decreasing on $\RR$, just like in our discussion of CVaR in Remark \ref{rmk:oce_risks}. With $c > 1$, Cressie-Read DRO criteria are not strictly speaking OCE criteria, since the expected value is wrapped within $(\cdot)^{1/c_{\ast}}$. That said, an argument analogous to that in the proof of Proposition \ref{prop:collapse_broad_ocelike} using monotonicity can be used to show the same result, i.e., that $\Herr$ is included in the solution set of $\textup{DRO}_{c,\varepsilon}(\loss_{\textup{01}}(\cdot))$.
\end{rmk}
\begin{rmk}[Related case: criteria based on Orlicz regret]\label{rmk:orlicz_risk}
Let $f: \RR_{+} \to \overbar{\RR}_{+}$ be a proper, lower semi-continuous convex function, satisfying $f(1)=0$, $f(0) < \infty$, and super-coercivity in that $f(u)/u \to \infty$ as $u \to \infty$. Letting $f^{\ast}$ denote the usual convex conjugate of $f$, namely $f^{\ast}(u) \defeq \sup_{v \in \RR} [uv - f(v)]$, in recent work by \citet{frohlich2023a}, a class of learning criteria are introduced with the form
\begin{align}\label{eqn:defn_orlicz_regret_risk}
\underline{\crit}_{f,\varepsilon}(\loss) \defeq \inf_{\theta, \sigma} \left[ \sigma \left( \varepsilon + \theta + \exx\left[ f^{\ast}\left(\frac{\loss}{\sigma} - \theta\right) \right] \right) \right],
\end{align}
where the infimum is taken over $\theta \in \RR$ and $\sigma > 0$, and $\varepsilon > 0$ is a parameter of the criterion. Note that $f^{\ast}$ is finite, convex, and monotonically increasing on $\RR$ \citep[Prop.~3.2]{frohlich2023a}. Using this strong monotonicity property, just as we did for OCE criteria with increasing $\rho$, we can prove that under the zero-one loss, $\underline{\crit}_{f,\varepsilon}(\loss_{\textup{01}}(\cdot))$-optimal decisions coincide with the $\err{\cdot}$-optimal decisions.
\end{rmk}
\begin{rmk}[Non-monotonic alternative: variantile]\label{rmk:variantile}
The variance of a random variable, say $\loss$, can be naturally expressed as the minimum value of the function $\exx(\loss-\theta)^{2}$, with the minimum taken with respect to $\theta \in \RR$. The dispersion around $\theta$ here is measured in a symmetric fashion, since the same function is used both when $\loss>\theta$ and $\loss \leq \theta$. A natural \emph{asymmetric} extension considers re-scaling each of these cases with $2(1-\tau)$ and $2\tau$ respectively, where $\tau$ is a free parameter such that $0 < \tau < 1$. Writing this explicitly as a learning criterion, we have
\begin{align}\label{eqn:defn_variantile}
\underline{\crit}_{\tau}(\loss) \defeq \min_{\theta \in \RR} 2 \exx\left[ \abs{\mathbf{1}_{\theta}(\loss)-\tau}(\loss-\theta)^{2} \right]
\end{align}
where $\mathbf{1}_{\theta}(\loss) \defeq \indic\{\loss \leq \theta\}$. The value of $\theta$ that minimizes the function shown on the right-hand side of (\ref{eqn:defn_variantile}) is well-known in the economics literature as the ``expectile,'' and while the residual $\underline{\crit}_{\tau}(\loss)$ itself has received less attention, it has appeared in recent work under names such as ``variancile'' and ``variantile'' \citep{frongillo2021a}. Using the latter term here, the \term{variantile} in (\ref{eqn:defn_variantile}) generalizes the variance, with $\tau = 1/2$ recovering the variance as a special case. While of a similar nature to the $\underline{\crit}_{\rho}$ criteria give in (\ref{eqn:defn_oce_like}), the critical difference here is that $(\cdot)^{2}$ is not monotonic on $\RR$. As a result, collapse in the sense of Proposition \ref{prop:surrogate_inclusion} and Remarks \ref{rmk:oce_risks}--\ref{rmk:orlicz_risk} is not guaranteed, but instead the solution set collapses into the set of \emph{either} minimizers or maximizers of $\err{\cdot}$, analogous to \S{\ref{sec:collapse_broad_fixed}}. See \S{\ref{sec:variantile_collapse}} for a more detailed demonstration of this.
\end{rmk}

\section{Relationship with Surrogate Losses}\label{sec:surrogates}

Our analysis and discussion of criterion collapse in the preceding section was centered around losses with a Bernoulli distribution, which typically arise when using the zero-one loss $\ell_{\textup{01}}$ defined in (\ref{eqn:defn_zeroone}). While this loss function is ubiquitous as a metric for \emph{evaluation}, during training we rarely use $\ell_{\textup{01}}$ to design a criterion to be optimized directly. It is far more common to introduce a surrogate function, say $\ell(h;x,y)$, that is more congenial to numerical optimization. This results in there being two distinct loss distributions for each candidate $h$, namely that of the binary error $\loss_{\textup{01}}(h) = \ell_{\textup{01}}(h;\rdv{X},\rdv{Y})$ and the surrogate loss $\loss(h) = \ell(h;\rdv{X},\rdv{Y})$. Say we introduce some new criterion map $\crit(\cdot)$. We already know from \S{\ref{sec:collapse}} how most typical criteria collapse under $\loss_{\textup{01}}(h)$, but these insights do not apply in general for $\loss(h)$, which may have a much more complicated distribution than a simple Bernoulli. Placing our focus on the \emph{surrogate} criterion $\crit(\loss(\cdot))$, there are two natural instances of ``collapse'' with respect to error probability minimizers that we can conceive of:
\begin{align}
\label{eqn:sur_collapse_meainingless}
\Herr & \subset \argmin_{h \in \HH} \crit(\loss(h))\\
\label{eqn:sur_collapse_unavoidable}
\argmin_{h \in \HH} \crit(\loss(h)) & \subset \Herr.
\end{align}
In some learning scenarios, the properties (\ref{eqn:sur_collapse_meainingless}) and (\ref{eqn:sur_collapse_unavoidable}) may of course be desirable. In the case of (\ref{eqn:sur_collapse_meainingless}), which is the surrogate version of the collapse notion seen in \S{\ref{sec:collapse}}, we have that it is possible to be optimal in $\crit(\loss(\cdot))$ without having to make a sacrifice in terms of error probability $\err{\cdot}$. As for (\ref{eqn:sur_collapse_unavoidable}), in the traditional setup where $\crit(\cdot) = \exx[\cdot]$, proving that (\ref{eqn:sur_collapse_unavoidable}) holds is the central goal of ``classification calibrated'' surrogate design \citep{bartlett2006b}.

On the other hand, the inclusion in (\ref{eqn:sur_collapse_unavoidable}) means that error probability minimization is \emph{unavoidable} (under $\crit(\cdot)$ and $\ell(\cdot)$), which could have unintended repercussions in terms of other metrics such as fairness or privacy, for which tradeoffs with accuracy are well-known \citep{menon2021a}. In \S{\ref{sec:surrogates_nolink}} we first give a simple example showing how both (\ref{eqn:sur_collapse_meainingless}) and (\ref{eqn:sur_collapse_unavoidable}) can fail to hold under a limited model $\HH$. In contrast, when $\HH$ is a highly expressive model, we show in \S{\ref{sec:surrogates_unavoidable}} that for most important (binary) surrogate losses and a large class of criteria (in particular the ``tilted risk'' criteria), average error optimizers are in fact unavoidable (i.e., (\ref{eqn:sur_collapse_unavoidable}) holds). We complement this in \S{\ref{sec:surrogates_flood_softad}} by showing how criteria underlying typical ``ascent-descent'' gradient-based algorithms can be used to avoid unintentional collapse into $\Herr$.

\subsection{Disentangling the error and surrogate loss}\label{sec:surrogates_nolink}

For the remainder of \S{\ref{sec:surrogates}}, we will focus on the binary classification case, where our random data takes the form $\rdv{Z} = (\rdv{X},\rdv{Y})$, taking values in $\XX \times \YY$ with $\YY = \{-1,1\}$. Let $\mathcal{S}$ denote a set of scoring functions $s: \XX \to \RR$, with decision set $\HH \defeq \{\sign(s(\cdot)) \colonset s \in \mathcal{S}\}$ based on $\mathcal{S}$, noting that for any $u \in \RR$, $\sign(u)$ is $1$ when $u \geq 0$, and is $-1$ when $u < 0$. In this setting, we show how under a model $\HH$ with limited expressive power, it is very easy to construct an example where both notions of ``collapse'' just mentioned cannot hold.

\begin{figure*}[t!]
\centering
\includegraphics[width=0.33\textwidth]{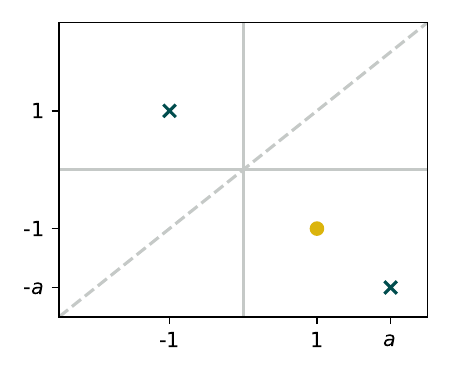}\includegraphics[width=0.33\textwidth]{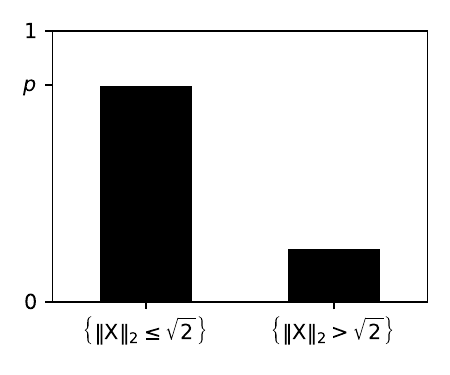}
\caption{In the left plot, we show the three possible data points that can arise in the example described in \S{\ref{sec:surrogates_nolink}}. Points above the dashed silver line are assigned a label of $1$ by $h_{1}$ and $-1$ by $h_{2}$; signs are reversed for all points below this line. For the outlying point in the bottom right, we have set $a=2$ in this example. In the right plot, we illustrate setting $p > 1/2$ to ensure the optimality of $h_{1}$ and $h_{2}$ in distinct criteria diverges.}
\label{fig:surrogate_nolink_data}
\end{figure*}

For simplicity and ease of visualization, consider an example on the plane, i.e., where $\XX = \RR^{2}$. Let $\rdv{X} = (\rdv{V}_{1}, \rdv{V}_{2})$, where $\rdv{V}_{1}$ and $\rdv{V}_{2}$ are real-valued random variables. Fixing an arbitrary $a > 1$, let $\rdv{V}_{1}$ take values in $\{-1,1,a\}$ and let $\rdv{V}_{2}$ take values in $\{-a,-1,1\}$. Assume that $\prr\{\rdv{X}=(-1,1)\} = \prr\{\rdv{X}=(1,-1)\} > 0$, and writing
\begin{align*}
p \defeq \prr\{\rdv{X}=(-1,1)\} + \prr\{\rdv{X}=(1,-1)\}
\end{align*}
let us say that $\prr\{\rdv{X}=(a,-a)\} = 1-p$ and $p < 1$ also hold. As for the labeling, let us assume
\begin{align}\label{eqn:surrogates_nolink_labels}
\rdv{Y} = \sign(\sqrt{2}-\Abs{\rdv{X}}_{2}) \sign(\rdv{V}_{2}-\rdv{V}_{1}).
\end{align}
Set $\mathcal{S} \defeq \{s_{1},s_{2}\}$, where $s_{1}(v_{1},v_{2}) \defeq v_{2}-v_{1}$ and $s_{2}(v_{1},v_{2}) \defeq v_{1}-v_{2}$ for all $v_{1},v_{2} \in \RR$. The resulting classifiers are $h_{1}(x) \defeq \sign(s_{1}(x))$ and $h_{2}(x) \defeq \sign(s_{2}(x))$ for $x \in \RR^{2}$. See Figure \ref{fig:surrogate_nolink_data} for a visualization of the three possible data points and the role of $p$. Error probabilities are $\prr\{h_{1}(\rdv{X}) \neq \rdv{Y}\} = 1-p$ and $\prr\{h_{2}(\rdv{X}) \neq \rdv{Y}\} = p$. When $p > 1/2$, clearly $\Herr = \{h_{1}\}$. As a typical surrogate loss, let us consider the usual binary ``logistic loss,'' namely the binary cross-entropy loss with scores converted to probabilities via the logistic function. Losses take the form $\ell_{\textup{log}}(s;x,y) \defeq \log(1+\exp(-ys(x)))$ for each $x \in \RR^{2}$ and $y \in \{-1,1\}$. Writing $\loss_{\textup{log}}(s) \defeq \ell_{\textup{log}}(s;\rdv{X},\rdv{Y})$, the surrogate loss distributions that result are quite simple; some arithmetic shows that we have
\begin{align*}
\loss_{\textup{log}}(s_{1}) & \in \{\log(1+\exp(-2)), \log(1+\exp(2a))\},\\
\loss_{\textup{log}}(s_{2}) & \in \{\log(1+\exp(-2a)), \log(1+\exp(2))\}
\end{align*}
with probability 1. For any fixed value of $a > 1$, it is always possible to take $0 < p < 1$ close enough to 1 that in terms of the \emph{expected} surrogate loss, we have
\begin{align*}
\argmin_{s \in \mathcal{S}} \exx[\loss_{\textup{log}}(s)] = \{s_{1}\}
\end{align*}
which is in agreement with the error probability minimizer. On the other hand, regardless of what value of $0 < p < 1$ is assumed, both the \emph{maximum} and \emph{minimum} losses incurred by $s_{2}$ are respectively smaller than those of $s_{1}$. That is,
\begin{align*}
\max_{\rdv{X},\rdv{Y}} \loss_{\textup{log}}(s_{2}) & < \max_{\rdv{X},\rdv{Y}} \loss_{\textup{log}}(s_{1}),\\
\min_{\rdv{X},\rdv{Y}} \loss_{\textup{log}}(s_{2}) & < \min_{\rdv{X},\rdv{Y}} \loss_{\textup{log}}(s_{1}).
\end{align*}
This simple observation means that criteria which tend to emphasize optimizing either the ``worst case'' or ``best case'' in terms of surrogate loss values will end up disagreeing with the error probability $\err{\cdot}$, and thus neither (\ref{eqn:sur_collapse_meainingless}) nor (\ref{eqn:sur_collapse_unavoidable}) can hold. This is not limited to the extreme case where the criterion map is $\crit(\cdot) = \max_{\rdv{X},\rdv{Y}}(\cdot)$ or $\crit(\cdot) = \min_{\rdv{X},\rdv{Y}}(\cdot)$. For example, analogous results can easily be shown to hold under the $\gamma$-tilted risk $\loss \mapsto (1/\gamma)\log(\exx[\mathrm{e}^{\gamma\loss}])$ (recalling Remark \ref{rmk:oce_risks}) for $\gamma \neq 0$ with $\abs{\gamma}$ sufficiently large, with $\gamma < 0$ emphasizing the best case, and $\gamma > 0$ emphasizing the worst case.

\subsection{Criteria that cannot avoid collapse}\label{sec:surrogates_unavoidable}

Let us assume that $\ell$ is a surrogate loss function which applies penalties to classification margins, i.e., assume that $\ell(s;x,y) \defeq \phi(ys(x))$, where margin penalizer $\phi: \RR \to \RR_{+}$ is assumed to be convex and non-negative over $\RR$, as well as differentiable at $0$; one typical example is $\ell = \ell_{\textup{log}}$ from \S{\ref{sec:surrogates_nolink}}. Let us also generalize beyond the expected loss criteria by considering the optimized certainty equivalent (OCE) criterion, defined for any $\loss \in \LL$ by
\begin{align}\label{eqn:defn_OCE}
\textup{OCE}(\loss) \defeq \inf_{\theta \in \RR} \left[ \theta + \exx[\rho(\loss - \theta)] \right]
\end{align}
and characterized by $\rho: \RR \to \RR$ satisfying the properties mentioned in Remark \ref{rmk:oce_risks}. Our running assumption is that for each $\loss \in \LL$, there is always a $\theta^{\ast} \in \RR$ such that $\textup{OCE}(\loss) = \theta^{\ast} + \exx[\rho(\loss - \theta^{\ast})]$ holds. We show below that for a large sub-class of OCE criteria and loss functions, criteria collapse in the sense of (\ref{eqn:sur_collapse_unavoidable}) often cannot be avoided.
\begin{prop}\label{prop:surrogate_inclusion}
Assume that the margin penalizer satisfies $\phi^{\prime}(0) < 0$, and that $\rho$ in (\ref{eqn:defn_OCE}) is increasing (not just non-decreasing) and differentiable. With $\HH \defeq \{\sign(s(\cdot)) \colonset s \in \mathcal{S}\}$, let $\mathcal{S}$ be the set of all measurable functions on $\XX$, and let $(s_{1},s_{2},\ldots)$ be a sequence from $\mathcal{S}$, with the resulting classifiers denoted by $h_{n}(\cdot) \defeq \sign(s_{n}(\cdot))$. Write $\loss_{\phi}(h) \defeq \ell(s_{h};\rdv{X},\rdv{Y})=\phi(\rdv{Y}s_{h}(\rdv{X}))$ for the losses induced by $\phi$. Taking $n \to \infty$, the following implication holds:
\begin{align*}
\textup{OCE}(\loss_{\phi}(h_{n})) \to \inf_{h \in \HH} \textup{OCE}(\loss_{\phi}(h)) \implies \err{h_{n}} \to \errMin,
\end{align*}
where $\errMin \defeq \inf_{h \in \HH} \err{h}$. In addition, minimizers of the OCE surrogate satisfy
\begin{align*}
\argmin_{h \in \HH} \textup{OCE}(\loss_{\phi}(h)) \subset \Herr.
\end{align*}
\end{prop}
\begin{rmk}[Classification calibrated surrogates]
Our assumptions on the margin penalizer $\phi(\cdot)$ in Proposition \ref{prop:surrogate_inclusion} amount to it being \term{classification calibrated} in the sense of \citet{bartlett2006b}, a basic property common in all popular surrogates for binary classification. Many choices of $\phi(\cdot)$ are allowed by Proposition \ref{prop:surrogate_inclusion}: in addition to monotonic losses such as the logistic loss $\phi(u) = \log(1+\exp(-u))$, exponential loss $\phi(u) = \exp(-u)$, and hinge loss $\phi(u) = \max\{0,1-u\}$, it also captures non-monotonic choices such as the quadratic loss $\phi(u) = (1-u)^{2}$, the ARC-X4 loss $\phi(u) = \abs{1-u}^{5}$ \citep{breiman1999a}, and shifted Huber-Catoni penalties \citep{holland2019b}.
\end{rmk}
\begin{rmk}[Strict monotonicity of $\rho$]
While the class of OCE criteria in its most general form just asks for $\rho$ to be non-decreasing, in Proposition \ref{prop:surrogate_inclusion} we require that it be strictly monotonic (increasing). Should a flat (zero slope) region be allowed to exist, it allows for $\textup{OCE}(\loss_{\textup{log}}(\cdot))$-optimal but $\err{\cdot}$-sub-optimal solutions to exist. This is akin to the ``positive steepness'' condition seen in related work on DRO \citep[Lem.~1]{hu2018a}. This means that the special case of CVaR, amounting to $\textup{OCE}(\cdot)$ with $\rho(u) = \max\{0,u\}/(1-\beta)$ and $0 < \beta < 1$, is excluded. On the other hand, the closely related $\gamma$-tilted risk, with $\rho(u) = (\mathrm{e}^{\gamma u}-1)/\gamma$ and $\gamma > 0$, is included. Under models with high expressive capacity, Proposition \ref{prop:surrogate_inclusion} suggests that methods such as tilted ERM may not always be a trustworthy choice for learning tasks where some evaluation metric of interest (e.g., fairness, interpretability) is at odds with error probability.
\end{rmk}

\subsection{Criteria that can avoid collapse}\label{sec:surrogates_flood_softad}

Here we consider alternatives to the class of criteria captured by the ``unavoidability'' result in Proposition \ref{prop:surrogate_inclusion}. To do so, we introduce minor changes to the functional form of the OCE criterion of (\ref{eqn:defn_OCE}), and highlight links to existing algorithms in the literature (see Remark \ref{rmk:surrogates_flood_softad}). Let $\widetilde{\rho}: \RR \to [0,\infty)$ be a continuous convex function which is coercive (i.e., $\abs{u} \to \infty$ implies $\widetilde{\rho}(u) \to \infty$), takes its minimum at 0 (assuming $\widetilde{\rho}(0) = 0$), and is strictly convex on a neighborhood of 0. We use the symbol $\widetilde{\rho}$ (instead of $\rho$ as in \S{\ref{sec:surrogates_unavoidable}}) because these assumptions imply that $\widetilde{\rho}$ cannot be monotonic on $\RR$. Let us denote the \term{inner} and \term{outer loss-restraining criterion maps} by $\crit_{\textup{inn}}(\cdot)$ and $\crit_{\textup{out}}(\cdot)$ respectively, defined for any $\loss \in \LL$ by
\begin{align}
\label{eqn:defn_Cinner}
\crit_{\textup{inn}}(\loss;\theta) & \defeq  \widetilde{\rho}(\exx[\loss]-\theta),\\
\label{eqn:defn_Couter}
\crit_{\textup{out}}(\loss;\theta) & \defeq  \exx[\widetilde{\rho}(\loss-\theta)].
\end{align}
See Figure \ref{fig:demo_valid_rho_rhotilde} for some examples of valid $\widetilde{\rho}$ compared with valid $\rho$ used for OCE criteria. Intuitively, $\crit_{\textup{inn}}(\loss;\theta)$ simply asks the learning algorithm to arrive at a loss distribution whose mean is ``close'' to some threshold $\theta$; when $\theta = 0$, $\crit_{\textup{inn}}(\loss;\theta)$ and $\exx[\loss]$ are equivalent (i.e., their solution sets coincide). In contrast, $\crit_{\textup{out}}(\loss;\theta)$ asks that the loss distribution be \emph{well-concentrated} near $\theta$. In both cases, the notions of ``closeness'' and ``concentration'' are broad, and can be asymmetric (e.g., heavier right tails than left tails are allowed) when $\widetilde{\rho}$ is. The following result shows how even under a highly expressive model as in \S{\ref{sec:surrogates_unavoidable}}, constraining the loss distribution by means of criteria such as $\crit_{\textup{inn}}$ and $\crit_{\textup{out}}$ lets us circumvent the unavoidability of Proposition \ref{prop:surrogate_inclusion}.
\begin{prop}\label{prop:surrogates_flood_softad}
Under the assumptions of Proposition \ref{prop:surrogate_inclusion}, assume that $\errMin = 0$ and the margin penalizer $\phi(\cdot)$ is monotonic (non-increasing) and unbounded above. Consider the non-trivial case where $\Herr \neq \HH$. Then, there exists $\theta > 0$ such that (\ref{eqn:sur_collapse_unavoidable}) cannot hold under either of the loss-restraining criteria, i.e., we have
\begin{align*}
\Herr \cap \argmin_{h \in \HH} \crit(\loss_{\phi}(h);\theta) = \emptyset
\end{align*}
for both $\crit = \crit_{\textup{inn}}$ and $\crit = \crit_{\textup{out}}$.
\end{prop}
\begin{rmk}[Links to algorithms in the literature]\label{rmk:surrogates_flood_softad}
While we have introduced the loss-restraining criteria $\crit_{\textup{inn}}$ and $\crit_{\textup{out}}$ in (\ref{eqn:defn_Cinner}) and (\ref{eqn:defn_Couter}) in a rather general form, plugging in concrete examples of $\widetilde{\rho}$ can be shown to align with specific learning algorithms from the literature. Perhaps the best-known is that of ``Flooding,'' as studied by \citet{ishida2020a}. By setting $\widetilde{\rho}(\cdot) = \abs{\cdot}$ and applying sub-gradient descent with step-size $\alpha \geq 0$ to the resulting $\crit_{\textup{inn}}(\loss(\cdot);\theta)$ for any differentiable loss $\loss$, we end up with an update rule
\begin{align}\label{eqn:flooding_update}
h_{t+1} = h_{t} - \alpha \sign(\exx[\loss(h_{t})]-\theta)\exx[\nabla\loss(h_{t})].
\end{align}
In practice, of course, the true expectation is replaced with the empirical mean over a finite sample. From (\ref{eqn:flooding_update}) it is clear that as the sequence $(h_{1},h_{2},\ldots)$ progresses, as long as the expected loss is above threshold $\theta$, vanilla gradient descent (on the expected loss) is run. It is only if the expected loss falls below this threshold that the sign flips, changing the update to gradient \emph{ascent}. The original intuition behind this approach was simply to avoid over-optimizing a surrogate loss function, but since the original paper, links between this method and sharpness-aware minimization \citep{foret2021a} have been established; see for example \citet{karakida2023a}. The hard condition for switching between ascent and descent was softened in recent work by \citet{holland2023bddflood}; their approach can be interpreted in our framework as using $\crit_{\textup{out}}$ rather than $\crit_{\textup{inn}}$, and replacing the absolute value function with $\widetilde{\rho}(\cdot) = \sqrt{(\cdot)^{2}+1}-1$. The resulting gradient descent update looks like
\begin{align}\label{eqn:softad_update}
h_{t+1} = h_{t} - \alpha \exx[\widetilde{\rho}^{\prime}(\loss(h_{t})-\theta)\nabla\loss(h_{t})].
\end{align}
This procedure is called ``softened ascent-descent'' (SoftAD) by the authors; note that the hard switch of $\sign(\cdot)$ in (\ref{eqn:flooding_update}) is replaced in (\ref{eqn:softad_update}) by a smooth switch $\widetilde{\rho}^{\prime}(\cdot)$, plus the modulation is applied in a per-point fashion before averaging, rather than after. Given that the Flooding and SoftAD methods just described are captured as special cases in Proposition \ref{prop:surrogates_flood_softad}, we shall pay particular attention to them in our empirical analysis to follow in \S{\ref{sec:empirical}}.
\end{rmk}
\begin{rmk}[Alternative approaches in the literature]\label{rmk:avoidance_literature}
In the context of avoiding collapse of DRO criteria (e.g., Theorem \ref{thm:collapse_DRO}), \citet[\S{4}]{hu2018a} consider constraining the uncertainty set (our $\mathcal{P}$) such that the underlying data distribution is effectively controlled by a simple discrete latent variable. This nice trick allows them to express their ``structural ARM'' criterion in terms of a weighted sum of expected zero-one loss and variance-like terms (e.g., their equation 18), which in principle makes it possible to avoid collapse. Another clever approach to avoiding such collapse (this time for CVaR; Theorem \ref{thm:collapse_CVaR}) is to consider the task of learning an ensemble of $k$ candidates as done by \citet[\S{3}]{zhai2021b}, i.e., choosing $\mathbf{h} \defeq (h_{1},\ldots,h_{k})$ with each $h_{j} \in \HH$ and a weight vector $\lambda \defeq (\lambda_{1},\ldots,\lambda_{k})$. The underlying \emph{loss function} is then a weighted sum of zero-one losses, namely $\ell_{\textup{01}}^{(k)}(\mathbf{h},\lambda;x,y) \defeq \sum_{j=1}^{k}\lambda_{j}\ell_{\textup{01}}(h_{j};x,y)$, and thus in principle we can avoid collapse since the distribution is no longer Bernoulli. While both of these results describe settings in which collapse \emph{need not} occur under $\loss = \loss_{\textup{01}}$, collapse is still possible, and there are not conditions for avoiding collapse with surrogates (as in our Proposition \ref{prop:surrogates_flood_softad}).
\end{rmk}

\section{Empirical Analysis}\label{sec:empirical}

\begin{figure*}[t!]
\centering
\includegraphics[width=0.5\textwidth]{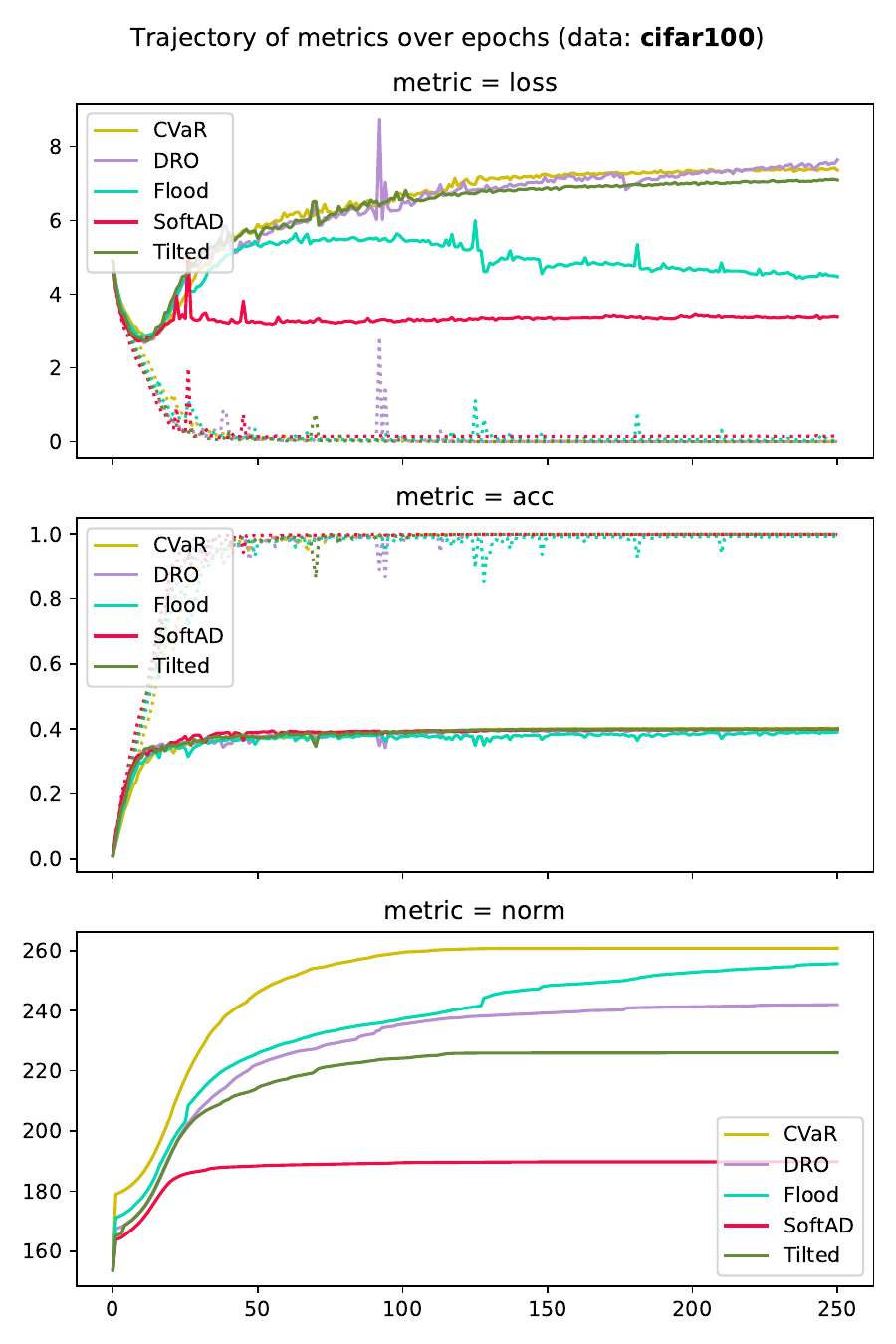}\includegraphics[width=0.5\textwidth]{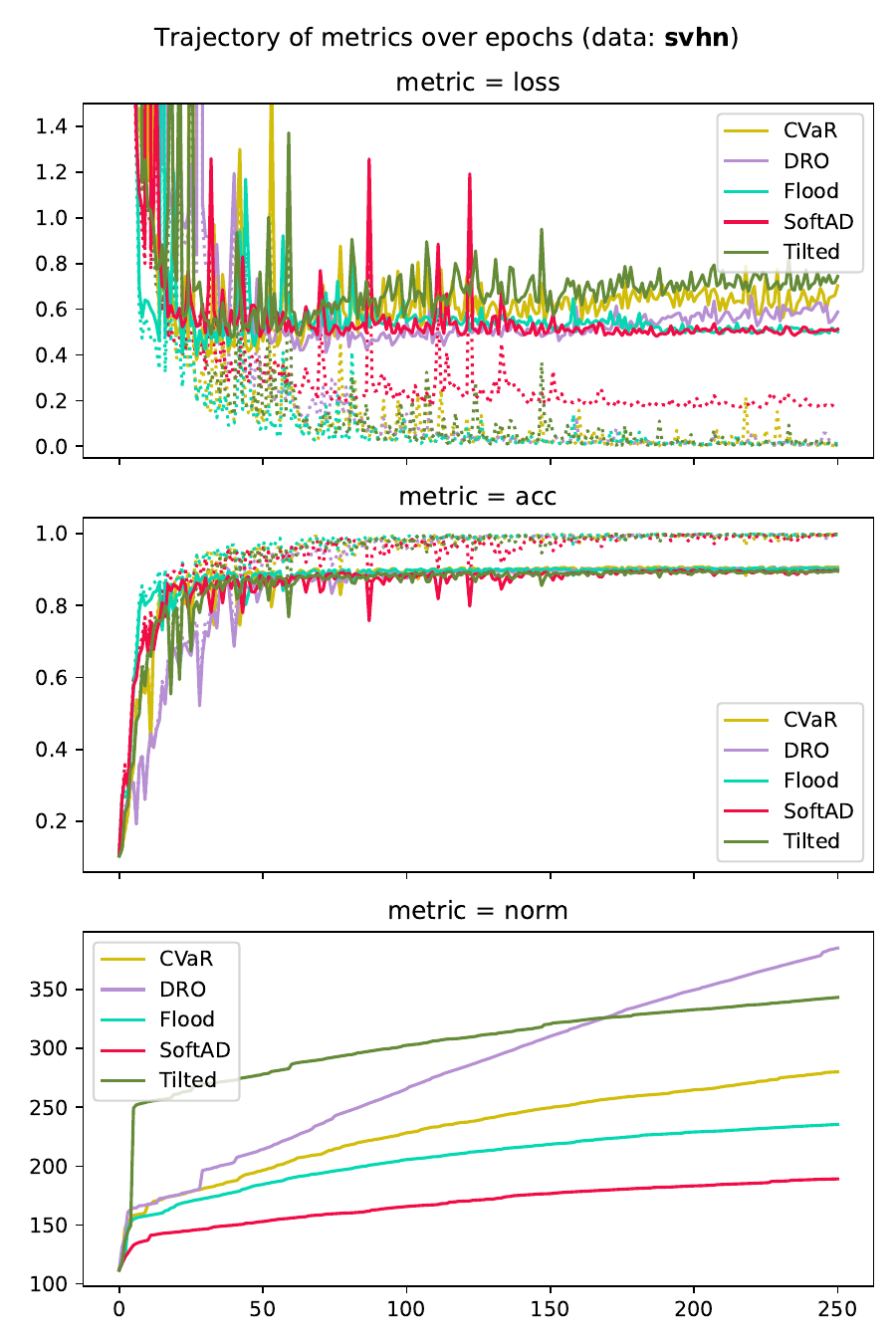}
\caption{Key metrics of interest (vertical axis) over epochs (horizontal axis). Here ``\texttt{loss}'' refers to average surrogate loss, ``\texttt{acc}'' refers to accuracy, and ``\texttt{norm}'' refers to the model L2 norm. Loss and accuracy are given for both training (dotted lines) and test data (solid lines). Plots on the left are for CIFAR-100, and plots on the right are for SVHN.}
\label{fig:benchmarks_1}
\end{figure*}

\begin{table*}[h!]
\begin{center}
\begin{tabular}{|c||c|c|c|c|}
\hline
 & CIFAR-10 & CIFAR-100 & FashionMNIST & SVHN \\
\hline\hline
CVaR & 0.14 (0.15) & 0.28 (0.13) & 0.46 (0.13) & 0.20 (0.12)\\
\hline
DRO & 0.02 (0.04) & 0.0 (0.0) & 0.0 (0.0) & 0.17 (0.20)\\
\hline
Flood & 0.05 (0.06) & 0.03 (0.05) & 0.01 (0.0) & 0.01 (0.0)\\
\hline
SoftAD & 0.06 (0.05) & 0.15 (0.13) & 0.01 (0.0) & 0.19 (0.19)\\
\hline
Tilted & 0.0 (0.0) & 0.0 (0.0) & 0.0 (0.0) & 0.0 (0.0)\\
\hline
\end{tabular}
\end{center}
\caption{Hyperparameter values selected for each method to maximize validation accuracy (averaged over trials). Standard deviation (again over trials) is given in parentheses.}
\label{table:benchmarks_hyperparams}
\end{table*}

Within our theoretical analysis of \S{\ref{sec:surrogates}}, we contrasted learning criteria that can and cannot avoid collapse into error probability minimizers. In particular, we emphasized the role of monotonicity in measuring the dispersion of losses around a threshold ($\rho$ vs.~$\widetilde{\rho}$). DRO, CVaR, and tilted ERM are typical monotonic methods, whereas Flooding and SoftAD are representative non-monotonic methods. To complement the points established in the previous section, here we implement the aforementioned methods and apply them to the training of non-linear neural network models for classification. Our interest is in \emph{tradeoffs}. Namely, having optimized for their respective learning criterion (under a base surrogate loss), we compare performance in terms of three alternative metrics, namely average surrogate loss, accuracy, and model complexity measured by the norm of model parameters.

\paragraph{Setup}
Our experiment design essentially follows the experimental design described by \citet{ishida2020a} (who first proposed the Flooding technique). The basic learning task is image classification from scratch. We use four standard datasets: CIFAR-10, CIFAR-100, FashionMNIST, and SVHN, accessed via the \texttt{torchvision} package. Each of these datasets has a default training-test split, which we leave fixed across all randomized trials. In each trial, however, we shuffle the training data and do an 80-20 split for training-validation. For CIFAR-10/100, we use ResNet-34; for FashionMNIST we flatten the images and use a simple feed-forward network (one hidden layer, 1000 units, batch normalization, ReLU activation); finally, for SVHN, we use ResNet-18. As an optimizer, we use vanilla SGD with step size 0.1 and momentum 0.9, run for 250 epochs, with mini-batch size of 200. In all cases, no pre-training is done; initial weights are randomly determined anew for each trial, and we note that the initial weights for any given trial are common across all methods being tested. The ``base'' surrogate loss function used for training is the multi-class cross-entropy loss.

\paragraph{Code}
A GitHub repository with code and seed values to re-create all the results presented in this paper is available at this URL: \url{https://github.com/feedbackward/collapse}.

\paragraph{Methods}
We are comparing the impact of different learning criteria, so the choice of ``method'' here amounts to choosing from the following: CVaR, DRO, Flooding, SoftAD, and tilted ERM. For DRO, we use Cressie-Read divergences (Remark \ref{rmk:dro_cressieread}), with $c=2$ fixed and $\varepsilon$ as a hyperparameter. For CVaR (cf.~Theorem \ref{thm:collapse_CVaR}), the quantile level $\beta$ is the hyperparameter. For both CVaR and DRO, we optimize the threshold $\theta$ simultaneously with model parameters. Tilted ERM has a closed form (\ref{eqn:tilted_risk_identity}) that we use, with hyperparameter $\gamma$. For the non-monotonic methods (Remark \ref{rmk:surrogates_flood_softad}), Flooding is $\crit_{\textup{inn}}$ with $\widetilde{\rho}(u) = \abs{u}$, and SoftAD is $\crit_{\textup{out}}$ with $\widetilde{\rho}(u) = \sqrt{u^{2}+1}-1$; both have threshold $\theta$ as a hyperparameter. For each method, we consider ten distinct hyperparameter candidates, as follows. CVaR: $\beta$ ranges between 0 and 0.9. DRO: with $\varepsilon = (1/(1-\widetilde{\varepsilon})-1)^{2}/2$, $\widetilde{\varepsilon}$ ranges between 0 and 0.9. Flooding: $\theta$ ranges between 0.01 and 1.0. SoftAD: $\theta$ ranges between 0.01 and 0.75. Tilted ERM: $\gamma$ ranges between 0 and 2.0. We remark that vanilla ERM is captured as a special case of the three monotonic methods ($\beta = 0$, $\varepsilon = 0$, and $\gamma = 0$).

\paragraph{Results}
For evaluation, before and after each epoch, we record three metrics: average surrogate loss, average accuracy (one minus average zero-one loss), and the L2 norm of model parameters. Here ``average'' refers to taking an average over each relevant data set (training, validation, test). Furthermore, we have run five independent trials, and the results to be presented are averaged over all trials. To choose a representative hyperparameter setting for each method, we select the value which maximizes the \emph{validation accuracy}. In Figures \ref{fig:benchmarks_1} and \ref{fig:benchmarks_2} we show the trajectory of our metrics of interest over epochs (with epoch 0 being the initial state). In Table \ref{table:benchmarks_hyperparams}, we show the hyperparameter values that were selected for each method (here also, averaged over trials).

\paragraph{Discussion}
There are several salient points that can be extracted from the results just described. First, it is clear that in terms of accuracy, all methods that we have tested include hyperparameter settings which can achieve near-perfect training accuracy, with test accuracy essentially the same across all methods. Note that for all trials and all datasets, tilted ERM is such that optimizing for validation accuracy leads to a setting of $\gamma = 0$, namely vanilla ERM. Second, while the ``best accuracy'' levels are essentially uniform across the methods tested, from the perspective of \emph{tradeoffs}, we see that things are far less uniform. In terms of average surrogate loss and model norm metrics, there is substantial dispersion between methods. We see that the two non-monotonic methods tend to achieve a far better average \emph{test} surrogate loss than the monotonic methods, with SoftAD being unique in terms of realizing a small model norm as well. Considering these trends (in Figures \ref{fig:benchmarks_1} and \ref{fig:benchmarks_2}) together with the selection trends (in Table \ref{table:benchmarks_hyperparams}), the results seem to suggest that even weakly constraining the surrogate loss distribution via $\crit_{\textup{inn}}$ and $\crit_{\textup{out}}$ (corresponding to Flooding and SoftAD here) looks to offer an appealing balance between the metrics we have studied here. Whether or not this behavior can be extended in a straightforward manner to metrics from other domains (e.g., fairness, prediction calibration, etc.) remains to be seen.

\section{Concluding Remarks}

In this work, we have formulated and studied the notion of criterion collapse in machine learning, with particular emphasis placed on settings in which \emph{unintended} collapse may occur, and how this phenomenon changes when surrogate functions are introduced into the picture. By highlighting situations in which unintentional collapse is possible, we are hoping to slowly but surely build a methodological roadmap of sorts, highlighting potential risks with model/algorithm design decisions, something analogous to the ``AI model cards'' that have attracted interest in recent years \citep{liang2024a}. By complementing our theory with experiments, we try to highlight how learning algorithms derived from distinct criteria classes (here the key difference is monotonicity) can lead to rather different algorithm behavior and learned model properties, both during and after training. The idea is to take these empirical insights together with the aforementioned basic principles to guide smart AI system design when our goals are more diverse than just ``maximize the accuracy on a big training set.''

\section*{Acknowledgements}

This work was supported by JST PRESTO Grant Number JPMJPR21C6, and a grant from the SECOM Science and Technology Foundation.

\clearpage

\appendix

\section{Additional Notes}

\paragraph{Tradeoffs between metrics}
The navigation between ``measures'' and their ``proxies'' is a well-known challenge in the field of public policy design, and has recently been highlighted as an important challenge by researchers at OpenAI.\footnote{\url{https://openai.com/research/measuring-goodharts-law}}

\paragraph{Surveys on learning criteria}
In recent years, several in-depth surveys looking at trends in criterion design for machine learning and related disciplines have appeared. See \citet{holland2023survey}, \citet{hu2023a}, and \citet{royset2022arxiv} for representative work.

\paragraph{OCE and DRO references}
See \citet{lee2020a} and \citet[\S{3.2}]{holland2023survey} for background on OCE criteria in machine learning, and \citet{bental2007a} for an authoritative initial reference. Tilted ERM with $\gamma > 0$ is a special case of OCE, but it should be noted that a more general notion of ``tilted ERM'' can be considered by using the right-hand side of (\ref{eqn:tilted_risk_identity}) with general $\gamma \neq 0$ \citep{li2021a}. For $\gamma < 0$ the link to OCE criteria breaks down (reducing the role of right-side tails instead of emphasizing them), but leads to an array of different applications \citep{li2023a}. For the well-known special case of CVaR, see \citet{rockafellar2000a} for the landmark paper that presents the form of CVaR used in our Theorem \ref{thm:collapse_CVaR}. See \citet{kashima2007a} and \citet{takeda2008a} for early links to machine learning. CVaR can be derived as a limiting case of DRO under a specific $f$-divergence (\citep[Ex.~3]{duchi2021a}). For more general results on DRO with uncertainty sets constructed using on $f$-divergences, see \citet{namkoong2016a} and \citet{duchi2021a}.

\paragraph{Tail-sensitive criteria}
The basic argument of \S{\ref{sec:surrogates_nolink}} can be extended well beyond the tilted ERM criterion to a wide variety of other criteria including many OCE risks and their ``inverted'' variants; see \citet{lee2020a} or \citet{holland2023survey} for some representative examples of such criteria.

\paragraph{References on surrogate design}
There is a very large literature on the topic of designing surrogate functions under the $\exx[\cdot]$ criterion. For binary classification, see for example \citet{bartlett2006b} and \citet{reid2010a,reid2011a}. For the multi-class setting, see \citet{williamson2016a}, and the references within.

\begin{rmk}[Coherent risks]
One of the most well-known and well-studied classes of risk functions is that of ``coherent risk'' criteria.\footnote{See \citet[\S{3}]{rockafellar2013a} for background.} Coherent risks are characterized by several properties, including a very weak notion of monotonicity, which simply requires that for any random losses $\loss$ and $\loss^{\prime}$, a criterion map $\crit(\cdot)$ satisfies the following:
\begin{align}\label{eqn:coherence_weak_monotonicity}
\loss \leq \loss^{\prime} \text{ (almost surely) } \implies \crit(\loss) \leq \crit(\loss^{\prime}).
\end{align}
While a very natural property in a general context, when we restrict ourselves to Bernoulli losses, with say $\loss = \loss_{\textup{01}}(h_{1})$ and $\loss^{\prime} = \loss_{\textup{01}}(h_{2})$ based on candidates $h_{1},h_{2} \in \HH$, this monotonicity property becomes rather vacuous. Just consider when the condition in (\ref{eqn:coherence_weak_monotonicity}) actually can hold. It can only be valid in two cases: either we have $\loss = \loss^{\prime}$ (almost surely) or $\loss < \loss^{\prime}$ (almost surely). Note that for $\loss < \loss^{\prime}$ to hold almost surely, we must have $\loss = 0$ and $\loss^{\prime} = 1$ almost surely. In any case, we are dealing with either constant random variables or identical distributions, so the notion of ``monotonicity'' only really becomes interesting for more complicated distributions.
\end{rmk}

\paragraph{Gaps between theory and practice?}
As discussed in \S{\ref{sec:surrogates_unavoidable}}, our Proposition \ref{prop:surrogate_inclusion} could be interpreted as a ``negative result'' showing that criteria like tilted ERM (with positive tilt) cannot avoid error probability minimization. On the other hand, such criteria have been applied with effect to classification under class imbalance learning problems (e.g., \citet[\S{7.2.4}]{li2023a}). We remark that there is no conflict between our results and such empirical results, since in practice, we only have limited expressive power and more importantly limited capability to optimize a typically non-convex objective (viewed as a function on $\HH$). When effective model capacity is constrained, it is easy to find settings in which the zero-one loss mean is not controlled by relevant surrogate loss criterion (e.g., our example from \S{\ref{sec:surrogates_nolink}}).

\paragraph{Notes on experiment results}
Looking at Table \ref{table:benchmarks_hyperparams}, note that for CIFAR-100 and FahsionMNIST, in all five trials, the hyperparameter setting for both DRO and Tilted ERM was zero, i.e., vanilla ERM was selected. In the case of FashionMNIST, the trajectories in Figure \ref{fig:benchmarks_2} (right plot) align perfectly, but for CIFAR-100 (Figure \ref{fig:benchmarks_1}, left plot), they do not; this is due to randomness inherent in the PyTorch implementation of ResNet-34 (no pre-training) that we have left uncontrolled.

\section{Proofs}

\subsection{Proofs of results given in the main text}

\begin{proof}[Proof of Proposition \ref{prop:collapse_quantiles}]
Fixing $h \in \HH$ and taking $0 < \beta \leq 1-\err{h}$, we know that for any choice of $x \geq 0$ we have $\prr{\{\loss_{\textup{01}}(h) \leq x\}} \geq \beta$, but this fails for any $x < 0$. Thus we have $\mathrm{Q}_{\beta}(\loss_{\textup{01}}(h)) = 0$ for any $\beta$ in this range. For larger probability levels, namely where $1-\err{h} < \beta \leq 1$, we know that $\prr{\{\loss_{\textup{01}}(h) \leq x\}} \geq \beta$ holds for any $x \geq 1$, but fails for any $x < 1$, giving us $\mathrm{Q}_{\beta}(\loss_{\textup{01}}(h)) = 1$ for this range. Taking these facts together, we have
\begin{align*}
\mathrm{Q}_{\beta}(\loss_{\textup{01}}(h)) =
\begin{cases}
1, & \text{ if } 1-\err{h} < \beta \leq 1 \\
0, & \text{ if } 0 < \beta \leq 1-\err{h}.
\end{cases}
\end{align*}
Now, consider the task of minimizing $\mathrm{Q}_{\beta}(\loss_{\textup{01}}(h))$ as a function of $h \in \HH$ given a fixed level $0 < \beta < 1$. Note that in all situations, we can always say that
\begin{align}\label{eqn:collapse_quantiles_1}
\Herr \subset \argmin_{h \in \HH} \mathrm{Q}_{\beta}(\loss_{\textup{01}}(h)).
\end{align}
This should be intuitively clear since achieving a small quantile amounts to achieving a sufficiently small error probability, though the probability need not be \emph{minimal}. For completeness, let us deal with the corner cases. If $\min_{h \in \HH} \err{h} > 1-\beta$, then regardless of what $h$ we choose, the value is $\mathrm{Q}_{\beta}(\loss_{\textup{01}}(h)) = 1$, and so every $h \in \HH$ is trivially ``optimal.'' The same trivial optimality arises if $\max_{h \in \HH} \err{h} \leq 1-\beta$, since all $h$ result in $\mathrm{Q}_{\beta}(\loss_{\textup{01}}(h)) = 0$. When neither of these conditions hold, choosing $h \in \HH$ such that the error probability is small enough gives us the desired minimum quantile value. In all cases, the inclusion (\ref{eqn:collapse_quantiles_1}) holds.
\end{proof}

\begin{proof}[Proof of Proposition \ref{prop:collapse_broad_ocelike}]
For convenience, let us write
\begin{align*}
\crit_{\rho}(\loss;\theta) \defeq \theta + \exx[\rho(\loss-\theta)]
\end{align*}
for any $\loss \in \LL$ and $\theta \in \RR$. Since our running assumption is that the minimum (in $\theta$) is achieved on $\RR$, we have that $\underline{\crit}_{\rho}(\loss) = \min_{\theta \in \RR}\crit_{\rho}(\loss;\theta)$. Let $h_{1}, h_{2} \in \HH$ be any candidates such that $\err{h_{1}} \leq \err{h_{2}}$. In addition, let $\theta_{1}$ and $\theta_{2}$ respectively minimize $\crit_{\rho}(\loss_{\textup{01}}(h_{1});\cdot)$ and $\crit_{\rho}(\loss_{\textup{01}}(h_{2});\cdot)$. With these assumptions in place, the following chain of inequalities holds:
\begin{align*}
\underline{\crit}_{\rho}(\loss_{\textup{01}}(h_{1})) & = \crit_{\rho}(\loss_{\textup{01}}(h_{1});\theta_{1})\\
& \leq \crit_{\rho}(\loss_{\textup{01}}(h_{1});\theta_{2})\\
& = \theta_{2} + \rho(-\theta_{2}) + \err{h_{1}}\left( \rho(1-\theta_{2}) - \rho(-\theta_{2}) \right)\\
& \leq \theta_{2} + \rho(-\theta_{2}) + \err{h_{2}}\left( \rho(1-\theta_{2}) - \rho(-\theta_{2}) \right)\\
& = \underline{\crit}_{\rho}(\loss_{\textup{01}}(h_{2})).
\end{align*}
Note that the second inequality uses the assumed monotonicity of $\rho$, which implies $\rho(1-\theta_{2}) \geq \rho(-\theta_{2})$. Since the choice of $h_{1}$ and $h_{2}$ was arbitrary, we can immediately conclude for any $h_{1},h_{2} \in \HH$ that
\begin{align*}
\err{h_{1}} \leq \err{h_{2}} \implies \underline{\crit}_{\rho}(\loss_{\textup{01}}(h_{1})) \leq \underline{\crit}_{\rho}(\loss_{\textup{01}}(h_{2})).
\end{align*}
Another immediate conclusion is that
\begin{align}\label{eqn:oce_like_inclusion}
\Herr \subset \argmin_{h \in \HH} \underline{\crit}_{\rho}(\loss_{\textup{01}}(h)).
\end{align}
If $\rho$ is monotonically \emph{increasing} (rather than just non-decreasing), then it follows that
\begin{align*}
\err{h_{1}} < \err{h_{2}} \implies \underline{\crit}_{\rho}(\loss_{\textup{01}}(h_{1})) < \underline{\crit}_{\rho}(\loss_{\textup{01}}(h_{2})).
\end{align*}
This means that one cannot be sub-optimal in $\err{\cdot}$ while still being optimal in $\underline{\crit}_{\rho}(\loss_{\textup{01}}(\cdot))$. As such, the two solution sets must coincide:
\begin{align}\label{eqn:oce_like_coincide}
\Herr = \argmin_{h \in \HH} \underline{\crit}_{\rho}(\loss_{\textup{01}}(h)).
\end{align}
Taking the conditions for (\ref{eqn:oce_like_inclusion}) and (\ref{eqn:oce_like_coincide}), we have the desired result.
\end{proof}

\begin{proof}[Proof of Proposition \ref{prop:surrogate_inclusion}]
To get started, we would like to express the OCE criterion in terms of a modified penalty applied to binary classification margins. Let us denote the modified penalty based on $\phi$ by
\begin{align}
\widetilde{\phi}(u;\theta) \defeq \theta + \rho(\phi(u)-\theta)
\end{align}
with the corresponding $\theta$-dependent expectation denoted by
\begin{align}
\crit_{\phi}(s;\theta) \defeq \exx[\widetilde{\phi}(\rdv{Y}s(\rdv{X});\theta)].
\end{align}
By our running assumptions, there always exists a minimizer of $\theta \mapsto \crit_{\phi}(s;\theta)$ on $\RR$, which for convenience we will denote by
\begin{align}
\theta^{\ast}(s) \in \argmin_{\theta \in \RR} \crit_{\phi}(s;\theta).
\end{align}
With this notation in hand, it follows that for each $s \in \mathcal{S}$ we have
\begin{align}\label{eqn:surrogate_inclusion_OCEoptimized}
\crit_{\phi}(s;\theta^{\ast}(s)) = \textup{OCE}(\phi(\rdv{Y}s(\rdv{X}))).
\end{align}
Moving forward, we will want to study sequences $(s_{1},s_{2},\ldots)$ of functions in $\mathcal{S}$ that minimize the OCE criterion in terms of the surrogate loss. More formally, let us assume $(s_{1},s_{2},\ldots)$ is any sequence satisfying
\begin{align}\label{eqn:surrogate_inclusion_minseq}
\lim\limits_{n \to \infty} \textup{OCE}(\phi(\rdv{Y}s_{n}(\rdv{X}))) = \inf_{s \in \mathcal{S}} \textup{OCE}(\phi(\rdv{Y}s(\rdv{X})))
\end{align}
With such a sequence in hand, note that there exists a value $\overbar{\theta} > 0$ such that for all integers $n \geq 1$, we have
\begin{align}\label{eqn:surrogate_inclusion_compact}
0 \leq \theta^{\ast}(s_{n}) \leq \overbar{\theta}.
\end{align}
The lower bound in (\ref{eqn:surrogate_inclusion_compact}) holds because $\phi(\cdot) \geq 0$, and thus the support of the surrogate losses $\phi(\rdv{Y}s(\rdv{X}))$ is bounded below by zero.\footnote{That this implies $0 \leq \theta^{\ast}(s)$ for any $s \in \mathcal{S}$ is a basic property of OCE criteria; see for example \citet[Prop.~2.1]{bental2007a}.} We cannot however guarantee that the support of $\phi(\rdv{Y}s(\rdv{X}))$ is bounded above. That said, $(s_{1},s_{2},\ldots)$ is not a completely arbitrary sequence; from (\ref{eqn:surrogate_inclusion_minseq}) we are assuming this sequence minimizes the OCE criterion. If the sequence $\theta^{\ast}(s_{n})$ were to grow without bound, it would contradict (\ref{eqn:surrogate_inclusion_minseq}).\footnote{To see this, just note that the $\theta$ term in the OCE criterion increases (with $\theta$) no slower than the $\rho(\cdot - \theta)$ term decreases (again, as $\theta$ grows), due to the slope, convexity, and monotonicity conditions on $\rho$.} As such, while the choice of $\overbar{\theta} < \infty$ may depend on the sequence, (\ref{eqn:surrogate_inclusion_minseq}) always implies (\ref{eqn:surrogate_inclusion_compact}).

Next, we would like to link up surrogate losses (under the OCE criterion) to the error probability $\err{\cdot}$. To do this, first note that the excess value in the OCE criterion can be bounded below by
\begin{align}
\nonumber
\textup{OCE}(\phi(\rdv{Y}s(\rdv{X}))) - \inf_{s \in \mathcal{S}}\textup{OCE}(\phi(\rdv{Y}s(\rdv{X}))) & = \crit_{\phi}(s;\theta^{\ast}(s)) - \inf_{s^{\prime} \in \mathcal{S}}\left[ \inf_{\theta \in \RR}\crit_{\phi}(s^{\prime};\theta) \right]\\
\label{eqn:surrogate_inclusion_firstineq}
& \geq \crit_{\phi}(s;\theta^{\ast}(s)) - \inf_{s^{\prime} \in \mathcal{S}} \crit_{\phi}(s^{\prime};\theta^{\ast}(s)).
\end{align}
Next let us note that the normalization of $\rho$ (slope 1, value 0 at the origin) combined with its convexity and monotonicity immediately implies $\rho(u) \geq u$ for all $u \in \RR$, and thus for any choice of $\theta \in \RR$, we have $\widetilde{\phi}(u;\theta) \geq \theta + (\phi(u)-\theta) = \phi(u)$. As such, it follows that
\begin{align}\label{eqn:surrogate_inclusion_nonneg}
\phi(\cdot) \geq 0 \implies \widetilde{\phi}(\cdot;\theta) \geq 0.
\end{align}
This fact is useful because it allows us to leverage structural results of \citet{bartlett2006b}. To do so, let us denote the so-called ``conditional $\widetilde{\phi}$-risk'' by
\begin{align}
C_{\beta}(u;\theta) \defeq \beta \widetilde{\phi}(u;\theta) + (1-\beta)\widetilde{\phi}(-u;\theta)
\end{align}
for each $u \in \RR$, $\theta \in \RR$, and $0 \leq \beta \leq 1$. Note that we can connect this function to our modified penalty through the equality
\begin{align}\label{eqn:surrogate_inclusion_condlink}
\crit_{\phi}(s;\theta) = \exx_{\rdv{X}}\left[ C_{\beta(\rdv{X})}(s(\rdv{X});\theta) \right]
\end{align}
that holds for any choice of $\theta \in \RR$ and $s \in \mathcal{S}$, where we have denoted $\beta(x) \defeq \prr\{\rdv{Y}=1 \cond \rdv{X}=x\}$. Optimal values of this function (with and without constraints) are denoted by
\begin{align}
H^{-}(\beta;\theta) & \defeq \inf \left\{ C_{\beta}(u;\theta) \colonset u(2\beta-1) \leq 0 \right\},\\
H(\beta;\theta) & \defeq \inf_{u \in \RR} C_{\beta}(u;\theta).
\end{align}
Here $H(\beta;\theta)$ corresponds to the optimal conditional risk computed in terms of $\widetilde{\phi}$, and $H^{-}(\beta;\theta)$ is used to quantify the best possible risk when a score (i.e., with $u$ as $s(x)$) disagrees with the Bayes optimal score (i.e., $(2\beta-1)$ with $\beta$ as $\prr\{\rdv{Y}=1 \cond \rdv{X}=x\}$). The so-called ``$\psi$-transform'' is based on the gap between these two optimal values. Denoting
\begin{align}
\widetilde{\psi}(u;\theta) \defeq H^{-}\left(\frac{1+u}{2};\theta\right) - H\left(\frac{1+u}{2};\theta\right)
\end{align}
for all $-1 \leq u \leq 1$, define $\psi(\cdot;\theta) \defeq \widetilde{\psi}^{\ast\ast}(\cdot;\theta)$, namely the Fenchel-Legendre biconjugate of $\widetilde{\psi}$. The first key structural results we wish to leverage are as follows:
\begin{align}\label{eqn:surrogate_inclusion_basicprops}
\psi(\cdot;\theta) \text{ is continuous on } [0,1], \quad \psi(\cdot;\theta) \geq 0, \quad \psi(0;\theta) = 0.
\end{align}
The properties in (\ref{eqn:surrogate_inclusion_basicprops}) hold for all $\theta \in \RR$.\footnote{These properties follow from Lemma 2 (parts 6, 7, and 8) of \citet{bartlett2006b} after applying (\ref{eqn:surrogate_inclusion_nonneg}).} Furthermore, using (\ref{eqn:surrogate_inclusion_nonneg}) and applying the proof of \citet[Thm.~1(1)]{bartlett2006b}, for any choice of $s \in \mathcal{S}$ and $\theta \in \RR$, we have
\begin{align}\label{eqn:surrogate_inclusion_errlink}
\psi\left(\err{h_{s}}-\errMin;\theta\right) \leq \exx_{\rdv{X}}\left[ C_{\beta(\rdv{X})}(s(\rdv{X});\theta) - H\left(\beta(\rdv{X});\theta\right) \right],
\end{align}
where $h_{s}(\cdot) \defeq \sign(s(\cdot))$, $\errMin \defeq \min_{h \in \HH} \err{h}$ and $\beta(x) \defeq \prr\{\rdv{Y}=1 \cond \rdv{X}=x\}$ as before. By assuming that $\mathcal{S}$ is the set of all measurable functions on $\XX$, we have
\begin{align}\label{eqn:surrogate_inclusion_optlink}
H\left(\beta(\rdv{X});\theta\right) = \inf_{s^{\prime} \in \mathcal{S}} \exx_{\rdv{X}}\left[ C_{\beta(\rdv{X})}(s^{\prime}(\rdv{X});\theta) \right]
\end{align}
for any choice of $\theta \in \RR$.\footnote{This fact is also used in the proof of \citet[Thm.~1(1)]{bartlett2006b}.} Combining (\ref{eqn:surrogate_inclusion_errlink}), (\ref{eqn:surrogate_inclusion_optlink}), and (\ref{eqn:surrogate_inclusion_condlink}) with the inequality (\ref{eqn:surrogate_inclusion_firstineq}) established earlier, we may conclude that
\begin{align}\label{eqn:surrogate_inclusion_mainineq}
\textup{OCE}(\phi(\rdv{Y}s(\rdv{X}))) - \inf_{s \in \mathcal{S}}\textup{OCE}(\phi(\rdv{Y}s(\rdv{X}))) \geq \psi\left(\err{h_{s}}-\errMin;\theta^{\ast}(s)\right)
\end{align}
holds for any choice of $s \in \mathcal{S}$. Applying this bound (\ref{eqn:surrogate_inclusion_mainineq}) to any sequence $(s_{1},s_{2},\ldots)$ satisfying (\ref{eqn:surrogate_inclusion_minseq}), it follows from the basic properties of (\ref{eqn:surrogate_inclusion_basicprops}) that
\begin{align}\label{eqn:surrogate_inclusion_zerolim}
\lim\limits_{n \to \infty} \psi\left(\err{h_{n}}-\errMin;\theta^{\ast}(s_{n})\right) = 0,
\end{align}
where we have denoted $h_{n}(\cdot) \defeq \sign(s_{n}(\cdot))$ for readability. All that remains is to analyze how the sequence $(\err{h_{1}},\err{h_{2}},\ldots)$ behaves.

In order to establish $\err{h_{n}} \to \errMin$, we need to utilize some additional properties of $\rho$. We have assumed that $\rho$ is differentiable, and can easily confirm that the first two derivatives of $\widetilde{\phi}(\cdot;\theta)$ are as follows:
\begin{align*}
\widetilde{\phi}^{\prime}(u;\theta) & = \rho^{\prime}(\phi(u)-\theta)\phi^{\prime}(u),\\
\widetilde{\phi}^{\prime\prime}(u;\theta) & = \rho^{\prime\prime}(\phi(u)-\theta)(\phi^{\prime}(u))^{2} + \rho^{\prime}(\phi(u)-\theta)\phi^{\prime\prime}(u).
\end{align*}
Since $\rho$ and $\phi$ are both assumed to be convex, and strict monotonicity of $\rho$ implies $\rho^{\prime}(\cdot) > 0$, we can immediately conclude that $\widetilde{\phi}^{\prime\prime}(\cdot;\theta) \geq 0$, and thus $\widetilde{\phi}(\cdot;\theta)$ is convex on $\RR$, for any choice of $\theta \in \RR$. Furthermore, since we have assumed $\phi^{\prime}(0) < 0$, it also follows that $\widetilde{\phi}^{\prime}(0;\theta) < 0$. This property along with convexity implies that we have
\begin{align}
\label{eqn:surrogate_inclusion_calibprop_1}
\psi(u;\theta) & > 0, \quad  0 < u \leq 1\\
\label{eqn:surrogate_inclusion_calibprop_2}
\psi(u;\theta) & = \widetilde{\phi}(0;\theta) - H\left(\frac{1+u}{2};\theta\right)
\end{align}
for any $\theta \in \RR$.\footnote{The properties (\ref{eqn:surrogate_inclusion_calibprop_1}) and (\ref{eqn:surrogate_inclusion_calibprop_2}) follow by applying Theorem 2 and Lemma 2(9) of \citet{bartlett2006b}.} Using the definitions of $\widetilde{\phi}$ and $H(\cdot;\theta)$ along with (\ref{eqn:surrogate_inclusion_calibprop_2}), we can write
\begin{align*}
\psi(u;\theta) = \rho(\phi(0)-\theta) - \inf_{v \in \RR}\left[ \left(\frac{1+u}{2}\right)\rho(\phi(v)-\theta) + \left(\frac{1-u}{2}\right)\rho(\phi(-v)-\theta) \right].
\end{align*}
Introducing the helper function
\begin{align*}
g(\theta,u;v) \defeq \left(\frac{1+u}{2}\right)\rho(\phi(v)-\theta) + \left(\frac{1-u}{2}\right)\rho(\phi(-v)-\theta),
\end{align*}
clearly we have $\psi(u;\theta) = \rho(\phi(0)-\theta) - \inf_{v \in \RR}g(\theta,u;v)$. Due to the monotonicity of $\rho$ and the convexity of both $\phi$ and $\rho$, it is easy to directly verify that $(\theta,v) \mapsto g(\theta,u;v)$ is convex on $\RR^{2}$, for any choice of $0 \leq u \leq 1$. In addition, partial minimization preserves convexity, so $\theta \mapsto \inf_{v \in \RR}g(\theta,u;v)$ is also convex and continuous.\footnote{See for example \citet[B.3.3 and B.1.3]{bertsekas2015ConvexOpt}.} As such, for any convergent sequence $(\theta_{1},\theta_{2},\ldots)$ and any choice of $0 \leq u \leq 1$, we have
\begin{align}\label{eqn:surrogate_inclusion_continuity}
\lim\limits_{n \to \infty}\psi(u;\theta_{n}) = \psi\left(u;\lim\limits_{n \to \infty}\theta_{n}\right).
\end{align}
With this fact established, let us say that under (\ref{eqn:surrogate_inclusion_minseq}), the sequence $\err{h_{n}}$ does not converge to $\errMin$. In such a case, there exists some $\varepsilon > 0$ such that
\begin{align*}
\limsup_{n \to \infty} \left[ \err{h_{n}} - \errMin \right] = \varepsilon.
\end{align*}
From the definition of $\limsup$, we can always find a convergent subsequence, denoted $(s_{1}^{\prime},s_{2}^{\prime},\ldots)$ such that $\err{h_{n}^{\prime}} - \errMin \to \varepsilon$, where $h_{n}^{\prime}(\cdot) \defeq \sign(s_{n}^{\prime}(\cdot))$. This implies that for all large enough $n$, we have $\err{h_{n}^{\prime}} - \errMin \geq \varepsilon/2 > 0$, and further that we have
\begin{align}\label{eqn:surrogate_inclusion_contra_1}
\psi(\err{h_{n}^{\prime}} - \errMin;\theta^{\ast}(s_{n}^{\prime})) \geq \psi(\varepsilon/2;\theta^{\ast}(s_{n}^{\prime})) > 0
\end{align}
for all such $n$.\footnote{The first inequality holds using the fact that monotonicity of $u \mapsto \psi(u;\theta)$ is implied by the properties of being minimized at zero (see (\ref{eqn:surrogate_inclusion_basicprops})) and convexity on $[0,1]$ for all $\theta \in \RR$; see Lemma 2(2) of \citet{bartlett2006b}. The second inequality follows from (\ref{eqn:surrogate_inclusion_calibprop_2}).} Noting that from (\ref{eqn:surrogate_inclusion_compact}), all elements of sequence $\theta^{\ast}(s_{n}^{\prime})$ must be included in a bounded interval (i.e., a compact set), there exists a sub-sequence of $(s_{n}^{\prime})$ that we denote as $(s_{1}^{\prime\prime},s_{2}^{\prime\prime},\ldots)$, such that $\theta^{\ast}(s_{n}^{\prime\prime})$ converges.\footnote{See for example \citet[Thm.~2.37]{rudin1976PMA3rd}.} Utilizing the continuity property (\ref{eqn:surrogate_inclusion_continuity}) and positivity (\ref{eqn:surrogate_inclusion_calibprop_1}), it follows that
\begin{align}
\lim\limits_{n \to \infty} \psi(\varepsilon/2;\theta^{\ast}(s_{n}^{\prime\prime})) =  \psi\left(\varepsilon/2; \lim\limits_{n \to \infty}\theta^{\ast}(s_{n}^{\prime\prime})\right) > 0.
\end{align}
This however would contradict (\ref{eqn:surrogate_inclusion_zerolim}), which we have already shown to be implied by any sequence satisfying (\ref{eqn:surrogate_inclusion_minseq}). As such, we can conclude that under (\ref{eqn:surrogate_inclusion_minseq}), we have $\err{h_{n}} \to \errMin$.

Finally, for our statement regarding the inclusion of solution sets, for the trivial case that no minimizer of $h \mapsto \textup{OCE}(\loss_{\phi}(h))$ exists, we have $\emptyset \subset \Herr$; let us consider the non-trivial case where a minimizer exists. Write such a minimizer by
\begin{align}\label{eqn:surrogate_inclusion_solexists}
s^{\ast} \in \argmin_{s \in \mathcal{S}} \textup{OCE}(\phi(\rdv{Y}s(\rdv{X}))).
\end{align}
Writing $h^{\ast}(\cdot) \defeq \sign(s^{\ast}(\cdot))$, if $\err{h^{\ast}} > \errMin$ were to hold, then via positivity (\ref{eqn:surrogate_inclusion_calibprop_1}), we would have a contradiction to the inequality (\ref{eqn:surrogate_inclusion_mainineq}). As such, $h^{\ast} \in \argmin_{h \in \HH} \err{h}$ must hold, concluding the proof.
\end{proof}

\begin{proof}[Proof of Proposition \ref{prop:surrogates_flood_softad}]
To start, fix a threshold $\theta = \theta_{\varepsilon} \defeq \phi(0) + \varepsilon$ using some $\varepsilon > 0$ to be determined later. Next, note that for any $h \in \HH$, the expected surrogate loss can be written as
\begin{align}\label{eqn:surrogates_flood_softad_0}
\exx[\loss_{\phi}(h)] = \exx[\indic\{h(\rdv{X}) \neq \rdv{Y}\}\loss_{\phi}(h)] + \exx[\indic\{h(\rdv{X}) = \rdv{Y}\}\loss_{\phi}(h)].
\end{align}
Now, taking an arbitrary $h^{\ast} \in \Herr$ and the assumption that $\err{h^{\ast}} = 0$, it follows that
\begin{align}\label{eqn:surrogates_flood_softad_1}
\exx[\indic\{h^{\ast}(\rdv{X}) \neq \rdv{Y}\}\loss_{\phi}(h^{\ast})] = 0.
\end{align}
Combining (\ref{eqn:surrogates_flood_softad_0}) and (\ref{eqn:surrogates_flood_softad_1}) with the monotonicity of $\phi$, we have that the expected surrogate loss of $h^{\ast} \in \Herr$ can be bounded above by
\begin{align}\label{eqn:surrogates_flood_softad_2}
\exx[\loss_{\phi}(h^{\ast})] = \exx[\indic\{h^{\ast}(\rdv{X}) = \rdv{Y}\}\loss_{\phi}(h^{\ast})] \leq \phi(0).
\end{align}
With this in place, consider any $h \notin \Herr$, namely any $h \in \HH$ such that $\err{h} > 0$. Recall that by definition of $\HH$, there is some $s \in \mathcal{S}$ such that $h(x) = \sign(s(x))$ for all $x \in \XX$. Again by assumption, this $s(\cdot)$ is measurable.\footnote{We have a typical notion of Borel measurability in mind. It is sufficient, for example, if for any $a \in \RR$, we have that the event $\{s(\rdv{X}) \leq a\}$ is measurable, i.e., an element of the sigma-field composing the underlying measure space. See for example \citet[Ch.~1]{ash2000a} for more background.} Consider a simple re-scaling of $s$, namely for some constant $b > 0$, let us define
\begin{align*}
\widetilde{s}(x) \defeq
\begin{cases}
b, & s(x) \geq 0\\
-b, & s(x) < 0.
\end{cases}
\end{align*}
Using the measurability of $s(\cdot)$ and the fact that $\widetilde{s}(\cdot)$ is a simple function (in the usual measure theoretical sense), we have that $\widetilde{s} \in \mathcal{S}$ and thus defining $\widetilde{h}(x) \defeq \sign(\widetilde{s}(x))$, we have $\widetilde{h} \in \HH$ as well. In fact, since signs are not changed, we have $h(x) = \widetilde{h}(x)$ for all $x \in \XX$, and thus $\err{\widetilde{h}} = \err{h} > 0$. On the other hand, re-scaling impacts the margin penalties incurred via the surrogate loss, and the resulting expected value is
\begin{align}
\nonumber
\exx[\loss_{\phi}(\widetilde{h})] & = \exx[\indic\{\widetilde{h}(\rdv{X}) \neq \rdv{Y}\}\loss_{\phi}(\widetilde{h})] + \exx[\indic\{\widetilde{h}(\rdv{X}) = \rdv{Y}\}\loss_{\phi}(\widetilde{h})]\\
\label{eqn:surrogates_flood_softad_3}
& = \err{h}\phi(-b) + (1-\err{h})\phi(b).
\end{align}
Using the running assumption that $\phi(\cdot)$ is both unbounded and non-negative, from (\ref{eqn:surrogates_flood_softad_3}) it is clear that we can take $b > 0$ large enough that $\exx[\loss_{\phi}(\widetilde{h})] > \phi(0)$. Returning to the threshold $\theta_{\varepsilon}$, if we set $\varepsilon = \exx[\loss_{\phi}(\widetilde{h})] - \phi(0)$, we have found a value of $\theta_{\varepsilon}$ such that
\begin{align*}
\exx[\loss_{\phi}(h^{\ast})] \leq \phi(0) < \exx[\loss_{\phi}(\widetilde{h})] = \theta_{\varepsilon},
\end{align*}
noting that the first inequality comes from (\ref{eqn:surrogates_flood_softad_2}). Note that this implies
\begin{align*}
\exx[\loss_{\phi}(\widetilde{h})] - \theta_{\varepsilon} = 0 > \exx[\loss_{\phi}(h^{\ast})] - \theta_{\varepsilon}
\end{align*}
and using the local strict convexity of $\widetilde{\rho}$ near the origin, we have
\begin{align*}
\widetilde{\rho}(\exx[\loss_{\phi}(\widetilde{h})] - \theta_{\varepsilon}) = 0 < \widetilde{\rho}(\exx[\loss_{\phi}(h^{\ast})] - \theta_{\varepsilon}).
\end{align*}
In other words, we have that $h^{\ast} \in \Herr$ is sub-optimal in terms of $\crit_{\textup{inn}}(\loss_{\phi}(\cdot);\theta_{\varepsilon})$. Note that the preceding argument (in particular the choice of $b$) does not depend on the particular choice of $h^{\ast} \in \Herr$, and thus we can conclude that all elements of $\Herr$ share this sub-optimality, which is the desired result for $\crit = \crit_{\textup{inn}}$.

As for the remaining case of $\crit = \crit_{\textup{out}}$, the preceding argument can fail because a large $b$ used in defining $\widetilde{s}$ means that the distribution of $\loss_{\phi}(\widetilde{h})$ becomes arbitrarily widely spread out, and thus can be strongly penalized by $\crit_{\textup{out}}$. As a simple alternative, we consider re-scaling and flipping the sign of any $\err{\cdot}$-optimal classifier. Take any $h^{\ast} \in \Herr$, with $h^{\ast}(x) = \sign(s^{\ast}(x))$ for some $s^{\ast} \in \mathcal{S}$. Define a modified version of the scoring function as follows:
\begin{align*}
\widetilde{s}^{\ast}(x) \defeq
\begin{cases}
\frac{-s^{\ast}(x)}{\abs{s^{\ast}(x)}}, & s^{\ast}(x) \neq 0\\
0, & s^{\ast}(x) = 0.
\end{cases}
\end{align*}
As mentioned earlier, measurability is preserved, and so $\widetilde{s}^{\ast} \in \mathcal{S}$ and thus setting $\widetilde{h}^{\ast}(x) \defeq \sign(\widetilde{s}^{\ast}(x))$, we also have $\widetilde{h}^{\ast} \in \HH$. Since this classifier always disagrees with $h^{\ast}$, we have $1 = \err{\widetilde{h}^{\ast}} > \err{h^{\ast}} = 0$. With probability 1, we have
\begin{align*}
0 \leq \loss_{\phi}(h^{\ast}) \leq \phi(0) < \loss_{\phi}(\widetilde{h}^{\ast}) = \phi(-1).
\end{align*}
As such, by taking $\theta = \phi(-1)$, clearly we have that
\begin{align*}
\exx[\widetilde{\rho}(\loss_{\phi}(\widetilde{h}^{\ast})-\theta)] = 0 < \exx[\widetilde{\rho}(\loss_{\phi}(h^{\ast})-\theta)],
\end{align*}
where the inequality holds because $\loss_{\phi}(h^{\ast})-\phi(-1) < 0$ holds with probability 1, plus the strict convexity of $\widetilde{\rho}$ around the origin. This shows how while $h^{\ast}$ is optimal in $\err{\cdot}$, it is sub-optimal in terms of $\crit_{\textup{out}}(\loss_{\phi}(\cdot);\theta)$. Again, this approach holds for any choice of $h^{\ast} \in \Herr$, and thus the desired result holds for $\crit = \crit_{\textup{inn}}$ as well.
\end{proof}

\subsection{Divergence of CVaR from error probability}\label{sec:cvar_01_diverge}

As mentioned in Remark \ref{rmk:oce_risks}, CVaR satisfies the non-decreasing condition of Proposition \ref{prop:collapse_broad_ocelike}, but $\rho(u) = (u)_{+}/(1-\beta)$ is not strictly monotonic on $\RR$ (i.e., not increasing). With $\rho$ that is non-decreasing, we know the inclusion (\ref{eqn:oce_like_inclusion}) holds, telling us that all error probability minimizers are also optimal in terms of CVaR. Are there any CVaR-optimal solutions that are \emph{not} optimal in terms of error probability? The answer is that it depends on both $\HH$ and $\beta$. To consider this a bit more precisely, first note that quantiles (recall \S{\ref{sec:collapse_broad_quantiles}}) of the zero-one loss are also binary valued, namely that we have
\begin{align*}
\mathrm{Q}_{\beta}(\loss_{\textup{01}}(h)) =
\begin{cases}
1, & \text{ if } 1 > \beta > \prr{\{\loss_{\textup{01}}(h) = 0\}}\\
0, & \text{ if } 0 < \beta \leq \prr{\{\loss_{\textup{01}}(h) = 0\}}.
\end{cases}
\end{align*}
It is straightforward to check that the following implications hold:
\begin{align}
\label{eqn:cvar_model_too_poor}
\min_{h \in \HH} \err{h} > 1-\beta & \implies \mathrm{Q}_{\beta}(\loss_{\textup{01}}(h)) = 1, \text{ for all } h \in \HH\\
\label{eqn:cvar_model_too_good}
\max_{h \in \HH} \err{h} \leq 1-\beta & \implies \mathrm{Q}_{\beta}(\loss_{\textup{01}}(h)) = 0, \text{ for all } h \in \HH
\end{align}
When the model $\HH$ is poor enough (or $\beta$ is large/strict enough) that the condition in (\ref{eqn:cvar_model_too_poor}) holds, the CVaR risk is constant, i.e., $\underline{\crit}_{\rho}(\loss_{\textup{01}}(h)) = 1$ for all $h \in \HH$, and thus trivially all $h \in \HH$ are optimal in CVaR, but certainly need not be in $\err{\cdot}$. When the model is good enough (or $\beta$ is small/loose enough) that the condition in (\ref{eqn:cvar_model_too_good}) holds, we have for all $h \in \HH$ that
\begin{align*}
\rho(1-\mathrm{Q}_{\beta}(\loss_{\textup{01}}(h)))-\rho(-\mathrm{Q}_{\beta}(\loss_{\textup{01}}(h))) = \rho(1) > 0,
\end{align*}
and thus the strong monotonicity argument used earlier applies, telling us that CVaR-optimal solutions and error probability minimizers are identical (i.e., (\ref{eqn:oce_like_coincide}) holds).

\subsection{Collapse under variantile}\label{sec:variantile_collapse}

With Remark \ref{rmk:variantile} as context, note that for $0 \leq \theta \leq 1$, it is straightforward to confirm that under the zero-one loss $\loss = \loss_{\textup{01}}$, for any $h \in \HH$ we have
\begin{align*}
2\exx\left[ \abs{\mathbf{1}_{\theta}(\loss_{\textup{01}}(h))-\tau}(\loss_{\textup{01}}(h)-\theta)^{2} \right] = 2\left[ (1-\tau)(1-\err{h})\theta^{2} + \tau \err{h} (1-\theta)^{2} \right].
\end{align*}
Checking first-order conditions, this function can be minimized (in $\theta$) by setting
\begin{align*}
\theta = \frac{\tau \err{h}}{(1-\tau)(1-\err{h}) + \tau \err{h}}.
\end{align*}
Plugging this solution in to obtain an explicit form of $\underline{\crit}_{\tau}(\loss_{\textup{01}}(h))$, we have
\begin{align*}
& \underline{\crit}_{\tau}(\loss_{\textup{01}}(h)) =\\
& 2\left[ (1-\tau)(1-\err{h})\left(\frac{\tau \err{h}}{(1-\tau)(1-\err{h}) + \tau \err{h}}\right)^{2} + \tau \err{h} \left(\frac{(1-\tau)(1-\err{h})}{(1-\tau)(1-\err{h}) + \tau \err{h}}\right)^{2} \right].
\end{align*}
With this form in mind, simply replace $\err{h}$ with a variable $p$ that can be taken over $[0,1]$, and introduce a helper function
\begin{align}\label{eqn:nonmonotonic_variantile_helper}
g_{\tau}(p) \defeq 2\left[ (1-\tau)(1-p)\left(\frac{\tau p}{(1-\tau)(1-p) + \tau p}\right)^{2} + \tau p \left(\frac{(1-\tau)(1-p)}{(1-\tau)(1-p) + \tau p}\right)^{2} \right].
\end{align}
One can readily check that this function $g_{\tau}(\cdot)$ is strictly concave on the unit interval for any choice of $0 < \tau < 1$; see Figure \ref{fig:tests_nonmonotonic} for a numerical example. As a result, we may conclude that depending on the nature of $\HH$, the minimizer set of $\underline{\crit}_{\tau}(\loss_{\textup{01}}(\cdot))$ is either the set of minimizers \emph{or} maximizers of $\err{\cdot}$, a result analogous to the point raised in \S{\ref{sec:collapse_broad_fixed}}.

\begin{figure}[t]
\centering
\includegraphics[width=0.5\textwidth]{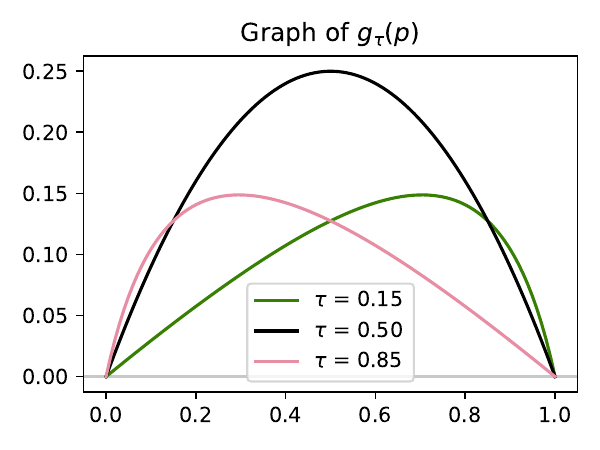}
\caption{Graphs of $g_{\tau}(\cdot)$ in (\ref{eqn:nonmonotonic_variantile_helper}) over the unit interval for varying choices of $\tau$.}
\label{fig:tests_nonmonotonic}
\end{figure}

\clearpage
\pagebreak

\section{Additional Figures}

As supplementary material for the main text, here we include two additional figures.
\begin{itemize}
\item Examples of valid $\rho$ for OCE criteria and $\widetilde{\rho}$ for loss-restraining criteria are shown in Figure \ref{fig:demo_valid_rho_rhotilde}.
\item In \S{\ref{sec:empirical}} we only gave metric trajectories for CIFAR-100 and SVHN; analogous results for CIFAR-10 and FashionMNIST are given in Figure \ref{fig:benchmarks_2}.
\end{itemize}

\begin{figure}[h]
\centering
\includegraphics[width=0.5\textwidth]{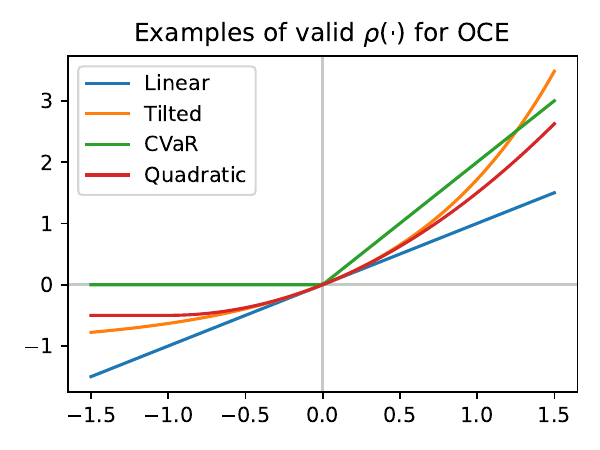}\includegraphics[width=0.5\textwidth]{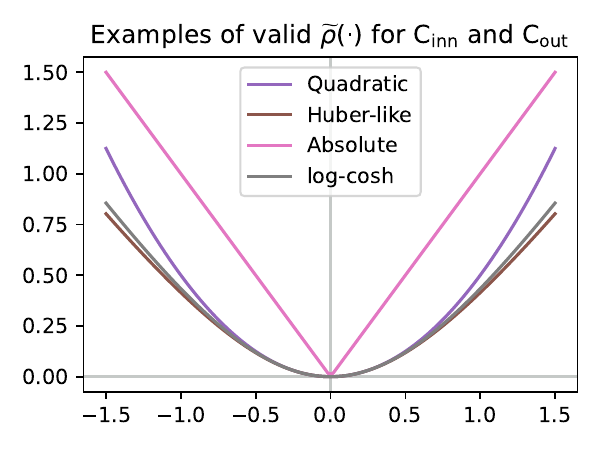}
\caption{Examples of valid choices of $\rho$ (left) and $\widetilde{\rho}$ (right) for use in defining OCE criteria (\ref{eqn:defn_OCE}) and loss-restraining criteria (\ref{eqn:defn_Cinner})--(\ref{eqn:defn_Couter}) respectively.}
\label{fig:demo_valid_rho_rhotilde}
\end{figure}

\begin{figure}[h]
\centering
\includegraphics[width=0.5\textwidth]{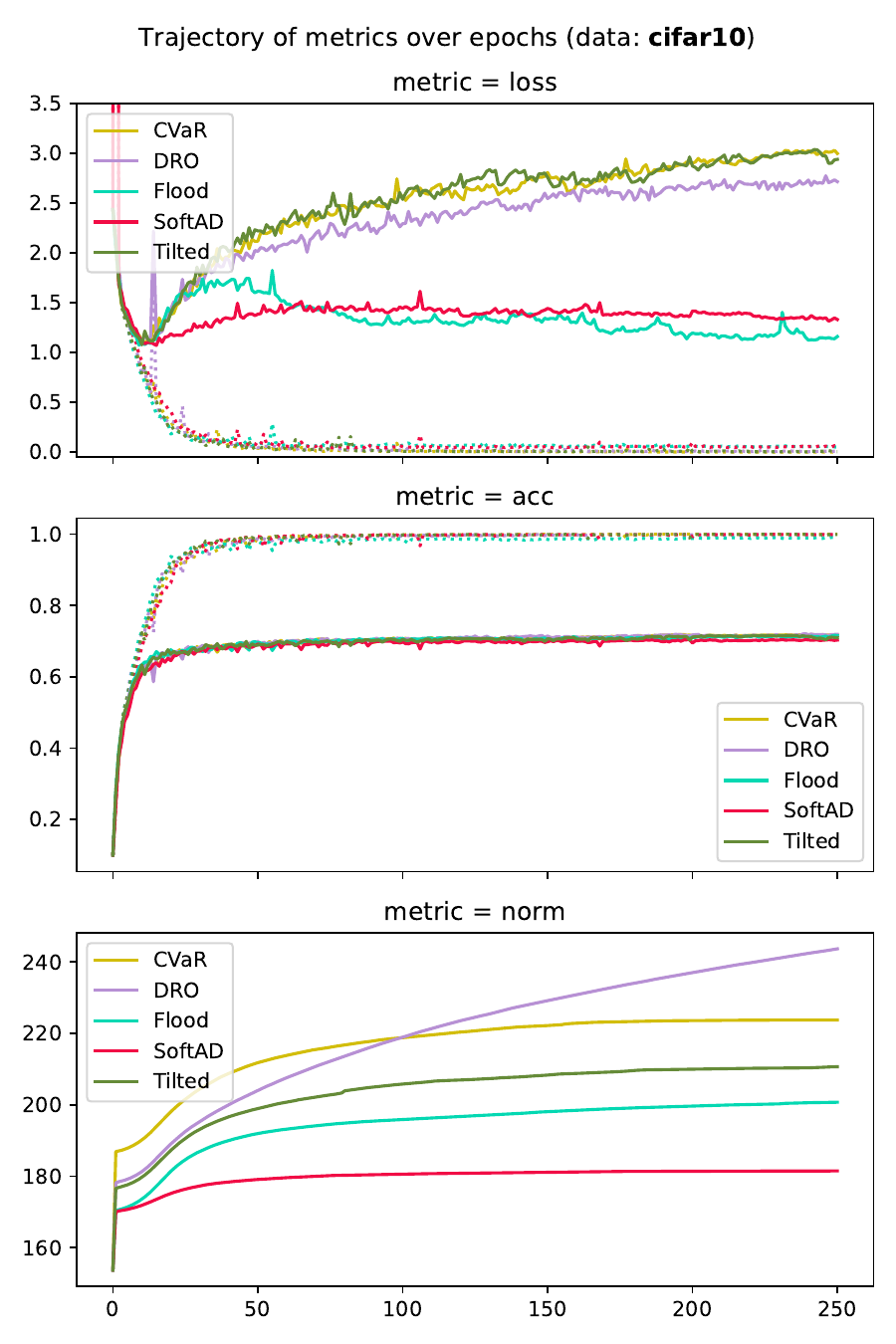}\includegraphics[width=0.5\textwidth]{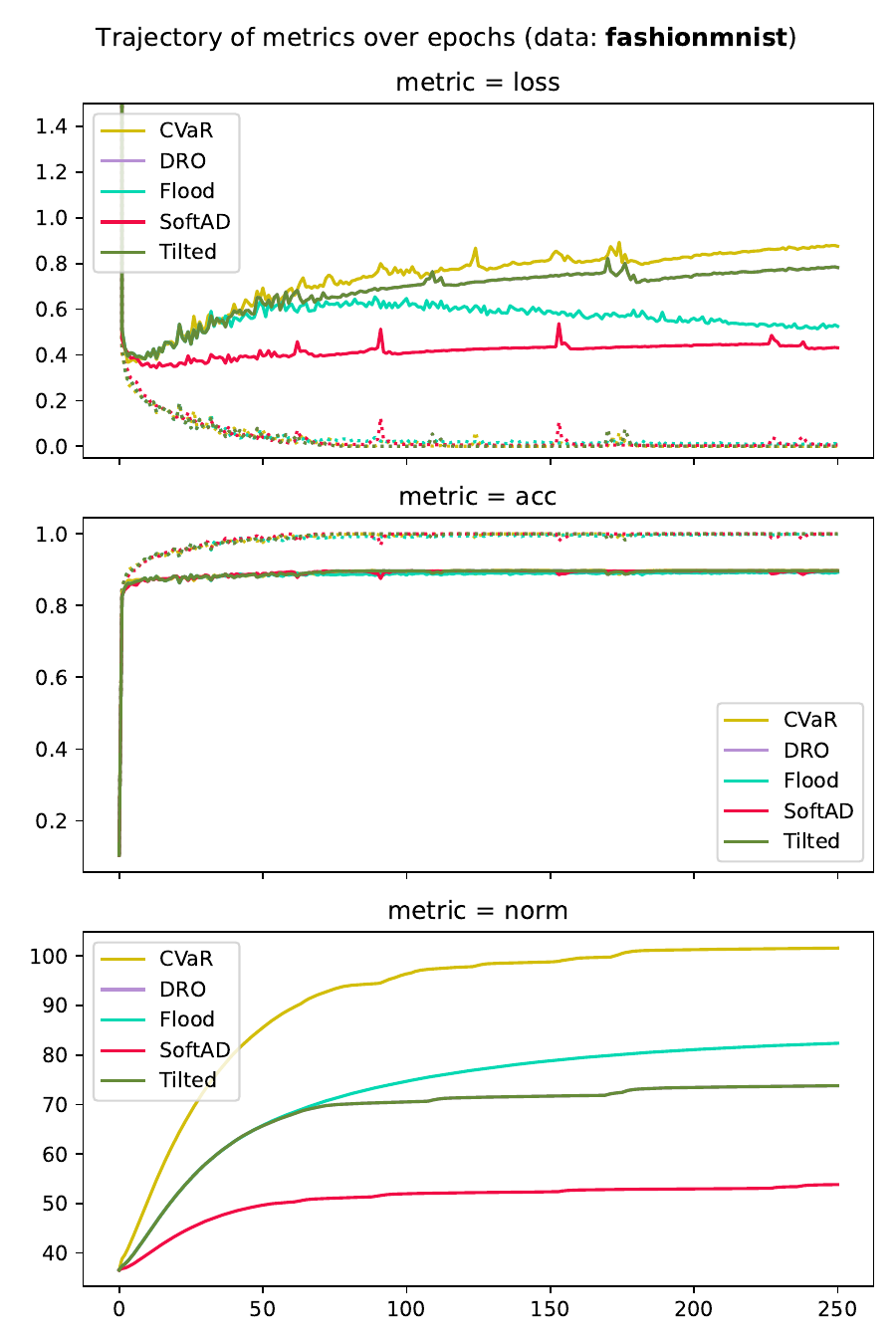}
\caption{Results for CIFAR-10 and FashionMNIST; see the caption of Figure \ref{fig:benchmarks_1} for details.}
\label{fig:benchmarks_2}
\end{figure}

\clearpage

\bibliography{../refs/refs}

\end{document}